%% file: main.tex
\documentclass[10pt,twocolumn,letterpaper]{article}

\usepackage{cvpr}              

\usepackage{graphicx}
\usepackage{amsmath}
\usepackage{amsthm}
\usepackage{amssymb}
\usepackage{booktabs}
\usepackage{bbm}
\usepackage{multicol}
\usepackage{multirow}
\usepackage{graphicx}
\usepackage{xcolor}
\usepackage[accsupp]{axessibility}  

\newtheorem{proposition}{Proposition}

\usepackage[pagebackref,breaklinks,colorlinks]{hyperref}

\usepackage[capitalize]{cleveref}
\crefname{section}{Sec.}{Secs.}
\Crefname{section}{Section}{Sections}
\Crefname{table}{Table}{Tables}
\crefname{table}{Tab.}{Tabs.}

\newcommand{\SO}{\mathrm{SO}(3)}

\makeatletter
\def\@fnsymbol#1{\ensuremath{\ifcase#1\or \dagger\or \ddagger\or
   \mathsection\or \mathparagraph\or \|\or **\or \dagger\dagger
   \or \ddagger\ddagger \else\@ctrerr\fi}}
\makeatother

\newcommand{\midsepdefault}{\aboverulesep = 0.2mm \belowrulesep = 0.2mm}
\midsepdefault

\def\cvprPaperID{3628} 
\def\confName{CVPR}
\def\confYear{2022}

\begin{document}

\title{FisherMatch: Semi-Supervised Rotation Regression via Entropy-based Filtering}

\author{Yingda Yin \qquad Yingcheng Cai \qquad He Wang\thanks{He Wang and Baoquan Chen are the corresponding authors
(\{hewang, baoquan\}@pku.edu.cn).
} 
\qquad Baoquan Chen\footnotemark[1]\\
Peking University\\
}
\maketitle

\input{tex/0_abstract}
\input{tex/1_intro}

\input{tex/2_related}
\input{tex/3_method}

\input{tex/4_experiment}

\input{tex/5_conclusion}

{\small
\bibliographystyle{ieee_fullname}
\bibliography{egbib}
}

\appendix
\input{supp}

\end{document}

%% file: tex/0_abstract.tex
\begin{abstract}

Estimating the 3DoF rotation from a single RGB image is an important yet challenging problem. Recent works achieve good performance relying on a large amount of expensive-to-obtain labeled data. 
To reduce the amount of supervision, we for the first time propose a general framework, FisherMatch, for semi-supervised rotation regression, without assuming any domain-specific knowledge or paired data. 
Inspired by the popular semi-supervised approach, FixMatch, we propose to leverage pseudo label filtering to facilitate the information flow from labeled data to unlabeled data in a teacher-student mutual learning framework.
However, incorporating the pseudo label filtering mechanism into semi-supervised rotation regression is highly non-trivial, mainly due to the lack of a reliable confidence measure for rotation prediction.
In this work, we propose to leverage matrix Fisher distribution to build a probabilistic model of rotation and devise a matrix Fisher-based regressor for jointly predicting rotation along with its prediction uncertainty. 
We then propose to use the entropy of the predicted distribution as a confidence measure, which enables us to perform pseudo label filtering for rotation regression. 
For supervising such distribution-like pseudo labels, we further investigate the problem of how to enforce loss between two matrix Fisher distributions.
Our extensive experiments show that our method can work well even under very low labeled data ratios on different benchmarks, achieving significant and consistent performance improvement over supervised learning and other semi-supervised learning baselines. Our project page is at  \href{https://yd-yin.github.io/FisherMatch}{https://yd-yin.github.io/FisherMatch}.

\end{abstract}

%% file: tex/1_intro.tex
\section{Introduction}
\label{sec:intro}

Incorporating deep neural networks to perform rotation regression is exerting an ever-important influence in computer vision, graphics and robotics. 
This is now one of the key technology in enabling a multitude of applications
such as camera relocalization and visual odometry\cite{bui20206d, gojcic2020learning}, object pose estimation and tracking\cite{xiang2017posecnn, weng2021captra}, and 6DoF robot grasping\cite{jiang2021synergies, breyer2020volumetric}.
One of the major obstacles to improving rotation regression is expensive rotation annotations.
Though many large-scale image datasets have been curated with sufficient semantic annotations, obtaining a large-scale real dataset with rotation annotations can be extremely laborious, expensive and error-prone\cite{xiang2014beyond}.
With the amount of labeled data being the bottleneck, there is a demand for methods that can leverage unlabeled data.

Regarding training models with fewer labels, semi-supervised learning (SSL) has been a powerful approach, mitigating the requirement for labeled data by providing a means of leveraging unlabeled data and thus attracting more and more attention.

Recent years have witnessed many processes in semi-supervised classification\cite{laine2016temporal, tarvainen2017mean, berthelot2019mixmatch, sohn2020fixmatch, gong2021alphamatch}, semi-supervised object detection\cite{liu2021unbiased, wang20213dioumatch}, and semi-supervised human and hand pose estimation\cite{pavllo20193d, kanaujia2007semi}. However, only few works address semi-supervised rotation regression, with most of them leveraging domain-specific knowledge, \textit{e.g.}, temporal smoothness of object pose\cite{liu2021semi} and strong assumptions, \textit{e.g.}, paired images from different viewpoints\cite{mariotti2021viewnet}.

The underlying reason for little work in this field is that rotation regression is very unique and challenging. First, it is undesirable to turn the rotation regression into a classification problem.
Given that 3D rotation space is continuous, discretizing the space into a small number of bins will lead to limited accuracy, which is intolerable for many applications involving rotation estimation.
Also, rotation regression is even not a standard regression problem. Given that rotation space $\SO$ is a non-Euclidean manifold\cite{zhou2019continuity}, a general regression algorithm needs to be tailored, taking the nonlinear structure of the rotation space into account. This further makes semi-supervised rotation regression a more challenging and less studied topic.

In this work, for the first time, we propose a general framework, namely \textit{FisherMatch}, for semi-supervised rotation regression. The problem we tackle is very general: \textit{using a neural network to regress rotation from a single RGB image}. Inspired by a popular semi-supervised learning approach, FixMatch \cite{sohn2020fixmatch}, initially developed for classification tasks, we attempt to process rotation regression problems in a similar flavor.

The key idea to the success of FixMatch is \textit{to filter out the pseudo labels with low classification confidence and only supervise the model outputs with highly confident labels}.
This mechanism ensures the quality of pseudo labels and thus significantly improves the performance of semi-supervised learning.
The underlying assumption is that the more confident a pseudo label is, the more closed this label is to the ground truth.
Or, in other words, this system needs to predict a confidence that can well indicate the correctness of its prediction.
Fortunately, a classification output naturally carries the information: the probability of its prediction can be used as its prediction confidence. We argue that the availability of such a reliable confidence measure is crucial to the success of FixMatch on semi-supervised classification tasks. Similarly, when adopting FixMatch to 3D object detection, 3DIoUMatch\cite{wang20213dioumatch} constructs a separate branch to predict the 3D IoU between the predicted bounding box and the ground truth bounding box as a localization confidence to filter out poor predictions. Although 3D IoU estimation is a regression task, 3DIoUMatch can move around the predicted bounding boxes as an augmentation trick, thus creating an infinite amount of training data for this confidence estimation module. This augmentation is crucial for such a confidence estimation module since the confidence estimation modules can only be trained using labeled data and must work on unlabeled data.

However, we argue that adopting FixMatch for rotation estimation is highly non-trivial. The biggest obstacle is how to estimate the prediction confidence for rotation regression.
For rotation regression, we don't have the largest probability from the bins as our confidence; also, for rotation estimation from a single RGB image, we can't perform such augmentation to change our rotation prediction; yet, we still need this uncertainty estimation module to work on unlabeled data with only training on a small set of labeled data. 

As pointed out by \cite{prokudin2018deep}, probabilistic modeling of rotation is the correct way to model the uncertainty of rotation regression. Parametric statistical methods for orientation statistics have long been established \cite{Downs1972orientation, khatri1977mises, jupp1979maximum, prentice1978invariant}.
In order to better resort to $\SO$ manifold which has a different topology than unconstrained values in $\mathbb{R}^N$, Deng \textit{et al.}\cite{deng2020deep} and Mohlin \textit{et al.}\cite{mohlin2020probabilistic} incorporate Bingham distribution and matrix Fisher distribution respectively to automatically learn uncertainties along with predictions, without further supervision. Thus, such networks can provide valuable information about the quality of the prediction. 
We prefer matrix Fisher distribution to Bingham distribution, since its rotation representation is continuous and its loss is convex with bounded gradient magnitudes, resulting in a stable training for neural networks \cite{levinson2020analysis, mohlin2020probabilistic}.

We thus devise a matrix Fisher-based rotation regressor that takes input a single RGB image and outputs the parameter of a matrix-Fisher distribution. 
Given the predicted distribution, we propose to use the entropy of this distribution as a confidence measure for pseudo label filtering. Basically, only pseudo labels with high confidence, \textit{i.e.} lower entropy than a threshold $\tau_\text{entropy}$, will pass the filtering and be used for supervising the model under training. 
Our experiment consistently proves that entropy is an efficient indicator of the prediction performance, 
not only in the case of 100 percent labeled data, but also in low data ratio cases down to 5 percent. 
Since FisherMatch outputs a distribution rather than a single rotation, our pseudo labels become a distribution, which requires research into the unsupervised loss enforced between two distributions. In this work, we investigate cross entropy loss and negative log likelihood loss, draw a connection between them, and find their proper usage in our experiments.

On common benchmark datasets of object rotation estimation from RGB images (ModelNet10-SO(3) and Pascal3D+) under various labeled data ratios, 
our experiment demonstrates a significant and consistent performance improvement over supervised learning and other semi-supervised learning baselines.

%% file: tex/2_related.tex
\section{Related Work}
\label{sec:related}

\paragraph{Rotation regression}
The choice of rotation representation is one of the core issues concerning rotation regression. The commonly used representations include Euler angles, axis-angles, unit quaternions, \textit{etc.} However, Euler angles suffer from gimbal lock, and quaternions have a double embedding giving rise to the existence of two disconnected local minima. Moreover, \cite{zhou2019continuity} argues that representations less than 4 dimensions are bound to have discontinuities and are difficult for neural networks to learn. To this end, the continuous 6D representation with Gram-Schmidt orthogonalization \cite{zhou2019continuity} and 9D representation with SVD orthogonalization \cite{levinson2020analysis} have been proposed respectively, leading to superior performance in rotation regression.

Several works propose to use probability distributions over rotations to further model prediction uncertainties along with rotation regression.
In Prokudin \textit{et al.} \cite{prokudin2018deep},  parameters of a mixture of Von Mises distribution using a biternion network are estimated. Deng \textit{et al.} \cite{deng2020deep} uses Bingham distribution over unit quaternions to jointly predict the rotation as well as the uncertainty. Estimation with matrix Fisher distribution \cite{mohlin2020probabilistic} learns to build the probability distribution over rotation matrices with unconstrained parameters. To further express arbitrary rotation distributions and better tackle rotation regression for symmetry objects, Implicit-PDF \cite{murphy2021implicit} chooses to represent the distributions implicitly by neural networks, instead of distribution parameters, where the $\SO$ space is uniformly discretized with the help of Hopf fibration \cite{yershova2010generating}.

\vspace{-2mm}
\paragraph{Semi-supervised classification} 
Semi-supervised learning is a long-studied field with a diversity of approaches, many in the field of classification. 
Consistency regularization and pseudo labeling are two measures with in-depth exploration. 
Consistency regularization was first proposed in \cite{bachman2014learning} which enforces the model to predict consistently across multiple perturbations \cite{laine2016temporal, tarvainen2017mean, ke2019dual, xie2020unsupervised}.
Pseudo labels \cite{lee2013pseudo} are artificial labels generated by the model itself and are used to further train the model, often applied along with a confidence-based thresholding to ensure the pseudo label quality.
Mixmatch \cite{berthelot2019mixmatch}, ReMixmatch \cite{berthelot2019remixmatch} and FixMatch \cite{sohn2020fixmatch} are holistic methods utilizing various augmentation and label sharpening strategies.

More recently, SimPLE \cite{hu2021simple} proposes the paired loss minimizing the statistical distance between confident and similar pseudo labels. SemCo \cite{nassar2021all} considers label semantics to prevent the degradation of pseudo label quality for visually similar classes in a co-training manner.
Dash \cite{xu2021dash} and FlexMatch \cite{zhang2021flexmatch} propose dynamic and adaptive pseudo label filtering, better suited for the training process.

\vspace{-2mm}
\paragraph{Semi-supervised regression}
Semi-supervised regression is a less-touched field compared with classification, where most of the works deal with regressing Euclidean variables, e.g., Parkinson's disease rating scales from multiple telemonitoring data in UCI repository \cite{asuncion2007uci}.
Early work of CoReg \cite{zhou2005semi} utilizes multiple k-nearest neighbor regressors with different distance metrics and leverages the predictions of one regressor to label the other regressors in a co-training manner. SSDKL \cite{jean2018semi} leverages the unlabeled data by minimizing predictive variance in the posterior regularization framework through the composition of neural networks and the probabilistic modeling of Gaussian processes.

\vspace{-2mm}
\paragraph{Self-/semi- supervised rotation estimation}
Several works tackle rotation estimation in a self-supervised manner. Mustikovela \textit{et al.} \cite{mustikovela2020self} leverages the analysis-by-synthesis technique that requires a lot of extra images for training a generative model. ViewNet \cite{mariotti2021viewnet} assumes the availability of paired data (same object, different poses).
The most relevant semi-supervised learning work is NVSM \cite{wang2021neural}, which shares the same assumptions on data and labels with us. In contrary to \textit{regression}, NVSM builds a category-level 3D cuboid mesh with feature vectors and estimates the object rotation in a render-and-compare technique through the distance-based rotation retrieval.
Less literature has been seen in the field of semi-supervised rotation regression. 
Mariotti \textit{et al.} \cite{mariotti2020semi} requires paired images of an object and enforces cross-reconstruction in an analysis-and-synthesis manner via rotating the encoded neural latent variables.

Our work draws insight from both the orientation statistics and the semi-supervised learning techniques introduced above, dedicated to correlating the techniques in two well-explored fields to tackle the problem in the general setting of semi-supervised rotation regression. 

%% file: tex/3_method.tex
\section{Method}
\label{sec:method}

\begin{figure*}[t]
\begin{center}
\includegraphics[width=0.8\linewidth]{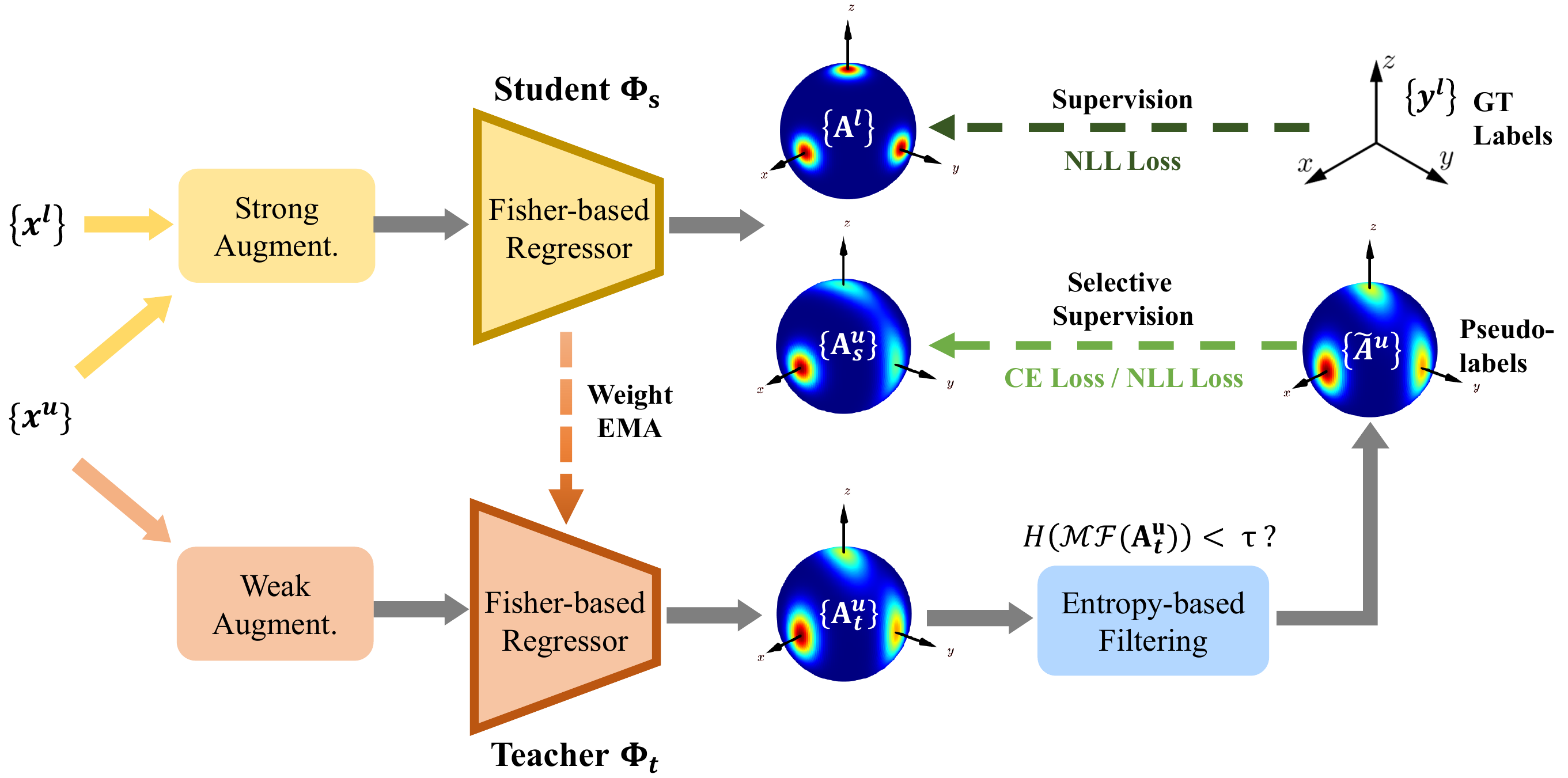}
\end{center}
\vspace{-5mm}
\caption{
\textbf{Pipeline overview.} Our matrix Fisher based-rotation regressor $\mathbf{\Phi}$ takes an RGB image $\boldsymbol{x}$ as input and outputs the parameter $\mathbf{A}$ of the predicted matrix Fisher distribution. We leverage a teacher-student mutual learning framework composed of a learnable student model and an exponential-moving-average (EMA) teacher model. On labeled data, the student network is trained by the ground-truth labels with the supervised loss; while on unlabeled data, the student model takes the pseudo labels from the EMA teacher. We leverage an entropy-based filtering technique to filter out noisy teacher predictions. The distribution visualization is borrowed from \cite{mohlin2020probabilistic} where $x$, $y$ and $z$ shown in black axes correspond to the standard basis of $\mathbb{R}^3$, and the pdf is shown on the sphere with a \textit{jet} color coding. See Appendix Section \ref{sec:supp_vis} for details of the visualization method.
}
\vspace{-2mm}
\label{fig:teaser}
\end{figure*}

In this work, we tackle the problem of learning to predict 3D object rotation from single RGB images under a semi-supervised setting,
where we have only a (small) set of labeled data $\left\{\boldsymbol{x}_{i}^{{l}}, \boldsymbol{y}_{i}^{l}\right\}_{i=1}^{N_{l}}$ and a larger set of unlabeled data $\left\{\boldsymbol{x}_{i}^{u}\right\}_{i=1}^{N_{u}}$. 
Here, $\boldsymbol{x}^{l}$ and $\boldsymbol{x}^{u}$ represent the labeled and unlabeled RGB image respectively, and $\boldsymbol{y}^{l}$ represents the ground-truth rotation in $\SO$ for a labeled data;
$N_{l}$ and $N_{u}$ are the number of labeled and unlabeled images, respectively.

Following a popular semi-supervised learning approach, FixMatch\cite{sohn2020fixmatch}, we adopt the teacher-student mutual learning framework, which we summarize in Section \ref{sec:revisit}.  In Section \ref{sec:prob_dist}, we make use of two probabilistic models of rotation for depicting the uncertainty in rotation prediction, namely Bingham distribution and matrix Fisher distribution\cite{deng2020deep,mohlin2020probabilistic}, and propose to use the entropy of the predicted matrix Fisher distribution as the prediction confidence for pseudo label filtering; In Section \ref{sec:loss}, for the purpose of enforcing loss between the teacher and the student, we construct two loss functions between pseudo labels and predicted distributions; Finally, in Section \ref{sec:protocol}, we introduce our training protocol in detail.

\subsection{Revisit FixMatch}
\label{sec:revisit}

The teacher-student mutual learning framework is a popular approach for semi-supervised learning. Mean Teacher\cite{tarvainen2017mean} proposes the first version, containing two jointly learned models - a \textit{teacher} and a \textit{student}.
The parameters of the teacher model are the exponential moving average (EMA) of the student model parameters that are updated by the stochastic gradient descent. The student model is trained by the ground-truth labels for the labeled data, and for the unlabeled data, the predictions of the teacher model serve as the \textit{pseudo labels} and are used to supervise the student network, through which, a history consistency is enforced between the two models. 

FixMatch \cite{sohn2020fixmatch} further develops this approach by proposing two strategies: asymmetric data augmentation and confidence-based pseudo label filtering. Asymmetric data augmentation means that the teacher model is fed by weakly augmented unlabeled samples while the student model takes strongly augmented unlabeled samples which contributes to the performance gap between the teacher and the student, facilitating correct information flow to the student.

Arguably, the most important contribution of FixMatch is to demonstrate the effectiveness of confidence-based pseudo label filtering. 
For a non-trivial semi-supervised learning task, previous works recognize that the pseudo labels generated by the teacher output suffer from significant noises\cite{wang20213dioumatch,sohn2020fixmatch}.
To this end, FixMatch proposes to filter out low-quality predictions and only supervise the student model with predictions with high confidence. 
This strategy avoids wrong supervision to the student model and has been proved to be very effective for challenging tasks, \textit{e.g.}, object detection\cite{liu2021unbiased, wang20213dioumatch}.
Given the difficulty of rotation regression, we further propose to leverage FixMatch as the basis of our framework for the rotation regression task.

\subsection{Probabilistic Modeling of Rotation}
\label{sec:prob_dist}
To model the uncertainty of rotation estimation, we leverage \textit{matrix Fisher} distribution to build a probabilistic model of rotation prediction, following Mohlin \textit{et al.}\cite{mohlin2020probabilistic}.

\textit{Matrix Fisher} distribution \cite{prentice1986orientation,khatri1977mises} $\mathcal{MF}(\mathbf{R};\mathbf{A})$ is a probability distribution over $\SO$ for rotation matrices, whose probability density function is in the form of 
\begin{equation}
p(\mathbf{R}) = \mathcal{{MF}}(\mathbf{R} ; \mathbf{A})=\frac{1}{F(\mathbf{A})} \exp \left(\operatorname{tr}\left(\mathbf{A}^{T} \mathbf{R}\right)\right)
\end{equation}
where parameter $\mathbf{A} \in \mathbb{R}^{3\times 3}$ is an arbitrary $3\times 3$ matrix and $F(\mathbf{A})$ is the normalizing constant. The mode and dispersion of the distribution can be computed from computing singular value decomposition of the parameter $\mathbf{A}$. Assume $\mathbf{A} = \mathbf{USV}^T$ and the singular values are sorted in descending order, the mode of the distribution is computed as 
\begin{equation}
\mathbf{\hat{R}}=\mathbf{U}\left[\begin{array}{ccc}
1 & 0 & 0 \\
0 & 1 & 0 \\
0 & 0 & \operatorname{det}(\mathbf{U} \mathbf{V})
\end{array}\right] \mathbf{V}^{T}
\end{equation}
and the singular values $\mathbf{S} = \text{diag}(s_1, s_2, s_3)$ indicates the strength of concentration. The larger a singular value $s_i$ is, the more concentrated the distribution is along the corresponding axis (the $i$-th column of mode $\mathbf{\hat{R}}$).

Another important probabilistic model for rotation is \textit{Bingham} distribution on $\mathcal{S}^3$ for unit quaternions. The probability density function is defined as

\begin{equation}
\mathcal{B}(\mathbf{q} ; \mathbf{M}, \mathbf{Z}) =\frac{1}{F(\mathbf{Z})} \exp \left(\mathbf{q}^{T} \mathbf{M} \mathbf{Z} \mathbf{M}^{T} \mathbf{q}\right)
\end{equation}
where $\mathbf{M} \in \text{O}(4)$ is a $4 \times 4$ orthogonal matrix and $\mathbf{Z} = \textrm{diag}(0, z_1, z_2, z_3)$ is a $4 \times 4$ diagonal matrix with $0\ge z_1 \ge z_2 \ge z_3$. The first column of parameter $\mathbf{M}$ indicates the mode and the remaining columns describe the orientation of dispersion while the corresponding $z_i, (i\in {1,2,3})$ describe the strength of the dispersion. $F(\mathbf{Z})$ is the normalizing constant.

It is well recognized that rotation matrix $\mathbf{R}$ and quaternion $\mathbf{q}$ are two different representations of rotation. 
Similarly, as discussed in \cite{prentice1986orientation}, matrix Fisher distribution and Bingham distribution are equivalent to each other differing only in parameterizations and rotation representations.
However, given that quaternion is not a continuous representation of rotation\cite{zhou2019continuity},  using matrix representation to learn a deep rotation estimation model has an intrinsic advantage and usually yields better performance. \cite{mohlin2020probabilistic} further shows that matrix Fisher distribution has a bounded gradient, which is favored by deep neural networks. 
Therefore, 9D rotation matrix is chosen as our representation, and matrix Fisher distribution is used for building our probabilistic rotation model.

\subsection{Entropy-based Pseudo Label Filtering}
Inspired by FixMatch, we only want the accurate predictions from the teacher model to ``teach'' the student model. Otherwise, noisy pseudo labels may slow down the training procedure, or even do harm to the whole process. 

For depicting the confidence of a predicted distribution, we propose to use \textit{entropy}, which is widely used in statistics acting as the degree of disorder or randomness in the system, as a measure of \textit{uncertainty}. A lower entropy generally indicates a more \textit{peaked} distribution which exhibits less uncertainty and higher confidence.  

In this work, we propose an entropy-based filtering mechanism leveraging the probabilistic modeling of the rotation estimation over $\SO$.
We devise a rotation regressor $\Phi$ that takes a single RGB image $\boldsymbol{x}$ and outputs the parameter $\mathbf{A}\in \mathbb{R}^{3\times3}$ of a matrix Fisher distribution 
\begin{equation}
    \mathbf{A}=\Phi(\boldsymbol{x}),
\end{equation}
which not only contains a predicted rotation as the \textit{mode} of this distribution, but also encode the information of the distribution \textit{concentration}. We then compute the entropy of this predicted distribution (see Equation \ref{eq:entropy}).

For \textit{pseudo label filtering}, we set a fixed entropy threshold $\tau$, and only reserve the prediction as a pseudo label if its entropy is lower than the threshold. Specifically, for unlabeled data $\boldsymbol{x}^u$, 
assume $p_t=\mathcal{MF}(\mathbf{A}^u_t)$ is the teacher output with $\mathbf{A}^u_t=\mathbf{\Phi}_t(\boldsymbol{x}^u)$ and
$p_s=\mathcal{MF}(\mathbf{A}^u_s)$ is the student output with $\mathbf{A}^u_s=\mathbf{\Phi}_s(\boldsymbol{x}^u)$ , the loss on unlabeled data is therefore:
\begin{equation}
\label{eq:unsuper}
L_u\left(\boldsymbol{x}^{u}\right) =  \mathbbm{1}  \left(H(p_t) \leq \tau\right) L\left(p_t, p_s\right)
\end{equation}
We discuss the loss function enforced between two distribution $L\left(p_t, p_s\right)$ in Section \ref{sec:loss}.

\subsection{Loss Function between Distributions}
\label{sec:loss}
For the labeled set $\left\{\boldsymbol{x}_{i}^{l}, \boldsymbol{y}_{i}^{l}\right\}_{i=1}^{N_{l}}$, we adopt the most common loss function, negative log likelihood (NLL) loss, to learn the probabilistic model of rotation, as in \cite{deng2020deep, mohlin2020probabilistic}. This loss minimizes the negative log likelihood of the ground-truth rotation in the predicted distributions, as shown below:
\begin{equation}
\begin{aligned}
\label{eq:nll_super}
L_l\left(\boldsymbol{x}^{l}, \boldsymbol{y}^{l}\right)&=-\log \left(\mathcal{MF}\left(\boldsymbol{y}^{l}; \mathbf{A}^{l})\right)\right) \\
\end{aligned}
\end{equation}
where $\mathbf{A}(\boldsymbol{x}^{l})$ denotes the network output fed with input $\boldsymbol{x}^{l}$.

For unlabeled data, both our network predictions and pseudo labels are distributions, and thus we need to enforce loss between two distributions, which is rarely the case for a regression problem.
We investigate two types of losses, 
\textit{i.e.}, negative log likelihood (NLL) loss and cross entropy (CE) loss. 

\noindent\textbf{Cross Entropy Loss $L^\text{CE}$} \ In classification problems, a widely-used loss function between two discrete distributions is cross entropy loss $L^\text{CE}$, whose gradient is equivalent to the gradient of KL divergence between two distribution \cite{glover2014quaternion}. We thus extend cross entropy loss $L^\text{CE}$ so as to enforce the consistency between pseudo labels and the student outputs:
\begin{equation}
\label{eq:ce_unsuper}
L^\text{CE}\left(p_t, p_s\right)=H\left(p_t, p_s\right)
\end{equation}

To compute $L^\text{CE}$ between two continuous distribution on $\SO$, we derive the analytical formula for the cross entropy between two matrix Fisher distributions $f\sim\mathcal{MF}(\mathbf{A}_f)$ and $g\sim\mathcal{MF}(\mathbf{A}_g)$, as shown below:

Assume $\mathbf{A}_f=\mathbf{U}_f\mathbf{S}_f \mathbf{V}_f^T$, $\mathbf{A}_g=\mathbf{U}_g\mathbf{S}_g \mathbf{V}_g^T$, $\gamma$ is the standard transform from unit quaternion to rotation matrix, $\mathbf{e}_i$ is the $i$-th column of $\mathbf{I}_4$, and $\mathbf{E}_i=\gamma({\mathbf{e}_i})$,
then we can derive
\begin{equation}
\small
\begin{split}
H(f, g)=\log F_g-\sum_{i=1}^{4} z_{gi}
\left(b_{i}^{2}+\sum_{j=1}^{4}\left(a_{i j}^{2}-b_{i}^{2}\right) \frac{1}{F_{f}} \frac{\partial F_{f}}{\partial z_{f j}}\right)
\end{split}
\end{equation}
\vspace{-4mm}
where 
{
\small
\begin{equation*}
z_{gi}=\operatorname{tr}(\mathbf{E}_i^T  \mathbf{S}_g) \qquad
z_{fj}=\operatorname{tr}(\mathbf{E}_j^T \mathbf{S}_j )
\end{equation*}
\begin{equation*}
a_{ij} = \gamma^{-1}(\mathbf{U}_f \mathbf{E}_i \mathbf{V}_f^T) \cdot \gamma^{-1}(\mathbf{U}_g \mathbf{E}_j \mathbf{V}_g^T) 
\end{equation*}
\begin{equation*}
b_{i} = \gamma^{-1}(\mathbf{U}_f \mathbf{E}_i \mathbf{V}_f^T) \cdot \gamma^{-1}(\mathbf{U}_g \mathbf{E}_i \mathbf{V}_g^T)
\end{equation*}
}and $F_f$ and $F_g$ are constant wrt. parameter $\mathbf{Z}$. See Appendix Section \ref{sec:supp_math} for the derivation. Note that when $f = g$, we can also get the entropy $H(f)$ for matrix Fisher distribution, as shown below:
\begin{equation}
\label{eq:entropy}
H(f)=\log F_f-\sum_{i=1}^{4} \left( z_{fi}  \frac{1}{F_{f}} \frac{\partial F_{f}}{\partial z_{fi}} \right)
\end{equation}

\noindent\textbf{NLL Loss $L^\text{NLL}$} \ Another option of the loss is to consider the negative log likelihood of the mode predicted by the teacher in the distribution predicted by the student, which is basically the NLL loss treating the teacher prediction as ground truth, as in the case of labeled data.

\begin{equation}
\label{eq:nll_unsuper}
\begin{aligned}
L^\text{NLL}\left(p_t, p_s\right)&=-\log p_s(\boldsymbol{y}_t^u),
\end{aligned}
\end{equation}
where $\boldsymbol{y}_t^u$ is the mode predicted by the teacher and can be computed by SVD of $\mathbf{A}_t^u$ (see Section \ref{sec:prob_dist}).

\noindent\textbf{Relationship between  $L^\text{NLL}$ and $L^\text{CE}$}
Here we intend to make connection between $L^\text{NLL}$ and $L^\text{CE}$. We find that $L^\text{CE}$ becomes $L^\text{NLL}$ when we decreases the dispersion of the distribution $p_t$ to a Dirac distribution $\delta(\mathbf{R}; \boldsymbol{y}_t^u)$ with its mode located at $\boldsymbol{y}_t^u$. 
We give a brief proof as below:
{
\small
\begin{equation*}
    \begin{aligned}
    L^\text{CE}\left(\text{Dirac}(p_t), p_s\right)&=
    H\left(\delta(\boldsymbol{y}_t^u), p_s\right)\\
     &=-\int_{\SO} \delta(\boldsymbol{y}_t^u)\log p_s \text{d}\mathrm{R} \\
    &=-\log p_s(\boldsymbol{y}_t^u) = L^\text{NLL}\left(p_t, p_s\right).
    \end{aligned}
\end{equation*}
}This exactly resembles the label sharpening technique used in semi-supervised classification\cite{berthelot2019mixmatch, sohn2020fixmatch}, where the teacher's output is either sharpened or turned into a hard label. To be specific, when we turn a predicted distribution $\mathcal{MF}(\mathbf{A}_{t}^u)$ into a hard label $\boldsymbol{y}_t^u$, $L^\text{CE}$ becomes $L^\text{NLL}$.
We use $L^\text{CE}$ in the experiments and investigate the different behavior of these two losses in Section \ref{sec:analysis}.

\subsection{Training Protocol}
\label{sec:protocol}
Our training is composed of two stages: a pre-training stage, where we train our rotation regressor on the labeled data, followed by an SSL stage where both the labeled and the unlabeled data are utilized. 
Our {matrix Fisher-based} rotation regressor is fed with an RGB image $\boldsymbol{x}$ and outputs a $3\times3$ matrix $\mathbf{A}$ as the predicted parameter of the matrix Fisher distribution. {We take the mode of the distribution as the predicted value.}

\noindent\textbf{Pre-training} \; We start with a supervised training procedure on the labeled set with the supervised loss as Eq. \ref{eq:nll_super}. 
We clone the rotation regressor to obtain a pair of teacher and student networks with the same initialization, once converged.

\noindent\textbf{Semi-supervised training} \;
In SSL stage, we utilize both the labeled data and the unlabeled data. A training batch contains a mixture of $\left\{\boldsymbol{x}_{i}^{l}\right\}_{i=1}^{B_{l}}$ labeled samples and $\left\{\boldsymbol{x}_{i}^{u}\right\}_{i=1}^{B_{u}}$ unlabeled samples. The loss function is composed of the supervised loss applied to the labeled samples and the unsupervised loss for the unlabeled samples
\begin{equation}
L=L_{l}\left(\boldsymbol{x}^{l},\boldsymbol{y}^{l}\right)+
\lambda_{u} L_{u}\left(\boldsymbol{x}^{u}\right)
\end{equation}
where $L_l$ is computed as Eq. \ref{eq:nll_super}, $L_u$ is as Eq. \ref{eq:ce_unsuper},
and $\lambda_u$ is the unsupervised loss weight.

In this stage, We adopt asymmetric augmentation and an exponential-moving-average teacher as stated in Sec. \ref{sec:revisit}.

%% file: tex/4_experiment.tex
\section{Experiment}
\label{sec:exp}

\input{tex/table_main1}
\input{tex/table_pascal}

\subsection{Datasets}

\textbf{ModelNet10-SO(3)}\cite{liao2019spherical} is created by rendering 3D models of ModelNet-10 \cite{wu20153d} that are rotated by uniformly sampled random rotations in $\SO$.
Following \cite{mohlin2020probabilistic,chen2021projective}, we focus on the \texttt{chair} and \texttt{sofa} category which exhibit the least rotational symmetries in the dataset. 
In the experiments, we set the ratio of labeled data as 5\% and 10\% of the training set.

\textbf{Pascal3D+}\cite{xiang2014beyond} contains real images from Pascal VOC and ImageNet of 12 rigid object classes. Following NVSM \cite{wang2021neural}, we evaluate 6 vehicle categories (\texttt{aeroplane}, \texttt{bicycle}, \texttt{boat}, \texttt{bus}, \texttt{car}, \texttt{motorbike}) which have relatively evenly distributed poses in azimuth angles, and set the number of labeled images as 7, 20 and 50 for each category respectively. We share the same selected 7 images as NVSM such that they are spread around the pose space. 

We follow the original train-test split and further divide the training split into the labeled set with ground truth and the unlabeled set without ground truth.

\subsection{Evaluation setup}

\paragraph{Baselines}

To the best of our knowledge,  we are the first to tackle semi-supervised rotation {regression} in this setting, hence the comparisons are made with self-made baselines. \textbf{Supervised-L1} uses a normal regressor and only trains on the labeled set with L1 loss {with the 9D-SVD\cite{levinson2020analysis} rotation representation}, while \textbf{Supervised-Fisher} uses our matrix Fisher regressor and also only go through the pretraining stage.
As an SSL baseline, \textbf{SSL-L1-Consistency} refers to adopting FixMatch into the task with the EMA teacher and asymmetric data augmentation preserved, but only applying L1 loss as the consistency supervision between the student and teacher predictions without filtering, due to lacking the confidence measure. Here, for non-Fisher regressors, we choose L1 instead of L2 loss, as \cite{deng2020deep} points out that L1 outperforms L2 for rotation regression.

We find the most relevant work to ours is {NVSM} \cite{wang2021neural}, which, though not regression-based, 
tackles the same task as ours and leverages a render-and-compare scheme through
distance-based rotation retrieval.  We borrow \textbf{NVSM} and their developed baselines as our compared baselines, including two supervised rotation estimation works (\textbf{StarMap}\cite{zhou2018starmap} and \textbf{NeMo}\cite{wang2021nemo}) and two standard classification networks (\textbf{Res50-Gene} and \textbf{Res50-Spec}), adapted into semi-supervised learning, respectively. Due to the unavailability of the training code
, we exactly follow the experiment settings of NVSM and evaluate on Pascal3D+ dataset. See Appendix Section \ref{sec:supp_settings} for more details.

\vspace{-2mm}
\paragraph{Evaluation metrics}
We evaluate the experiments by the mean error, the median error (in degrees) and the accuracy within 30$^\circ$ between the prediction and the ground truth.

\subsection{Results}
\begin{figure}[t]
	\begin{center}
		\begin{tabular}{cc}
			\hspace{-4mm} \includegraphics[width=0.48\linewidth]{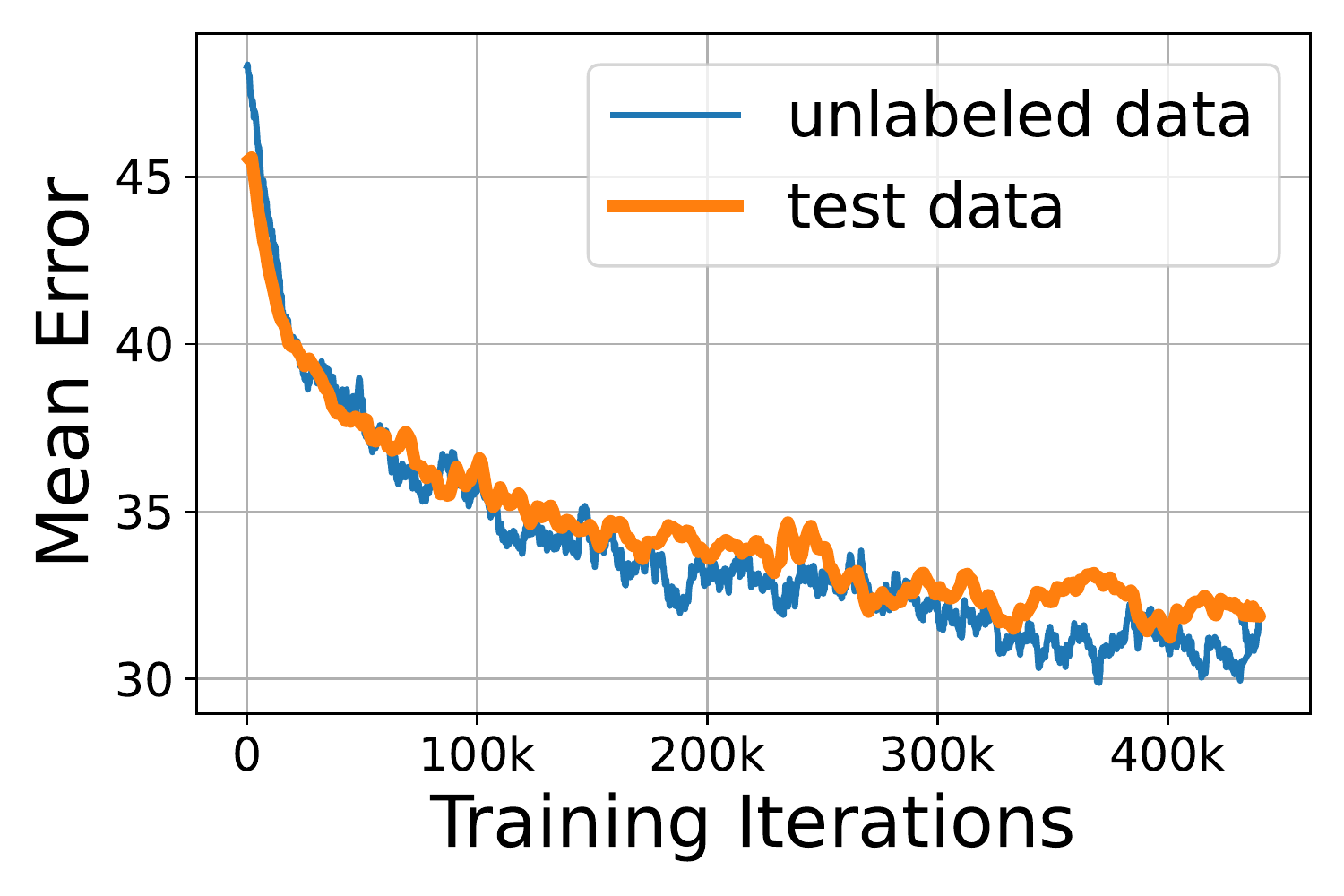} \hspace{-2mm}
			& \includegraphics[width=0.48\linewidth]{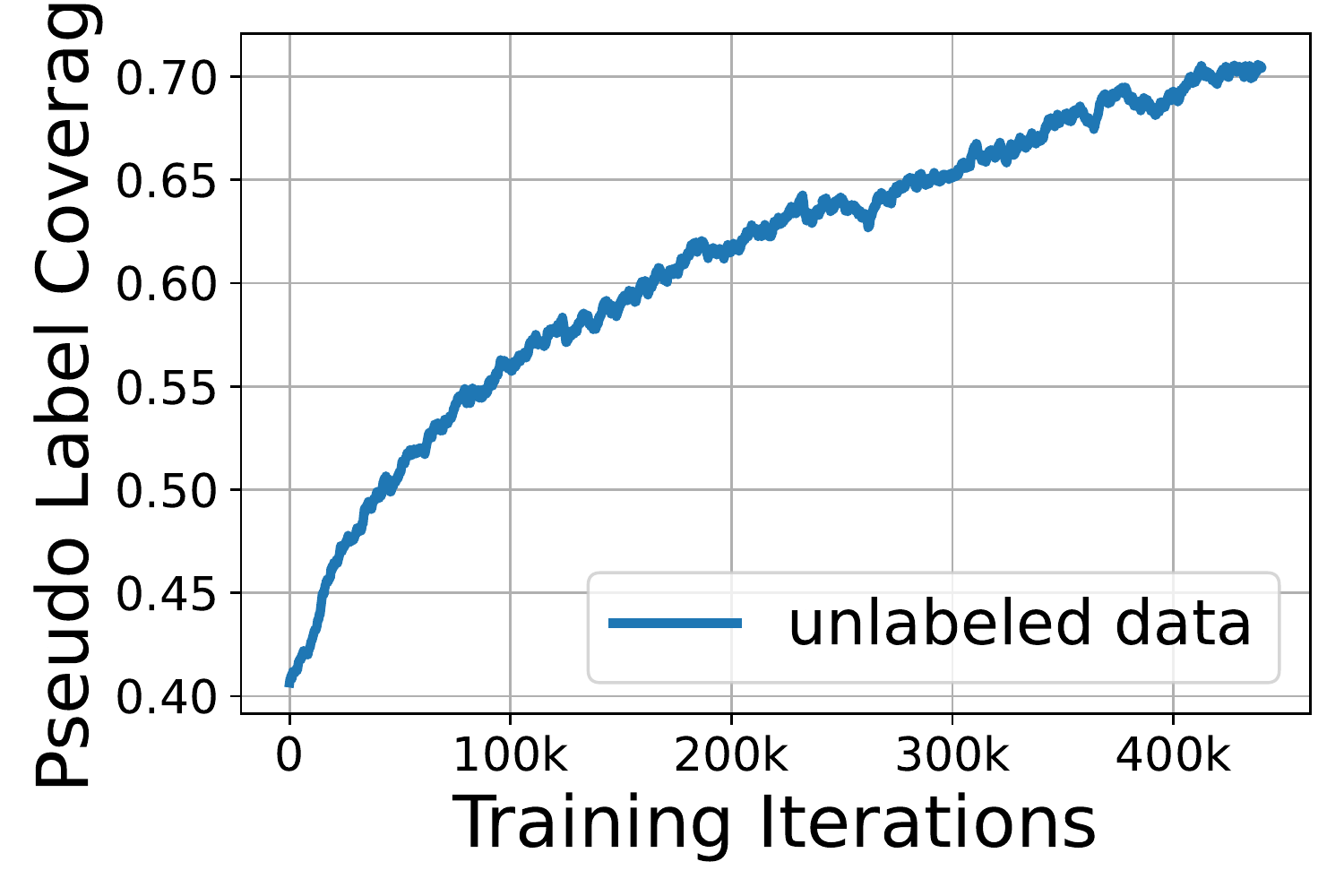}
			\\
			\hspace{-4mm} \includegraphics[width=0.48\linewidth]{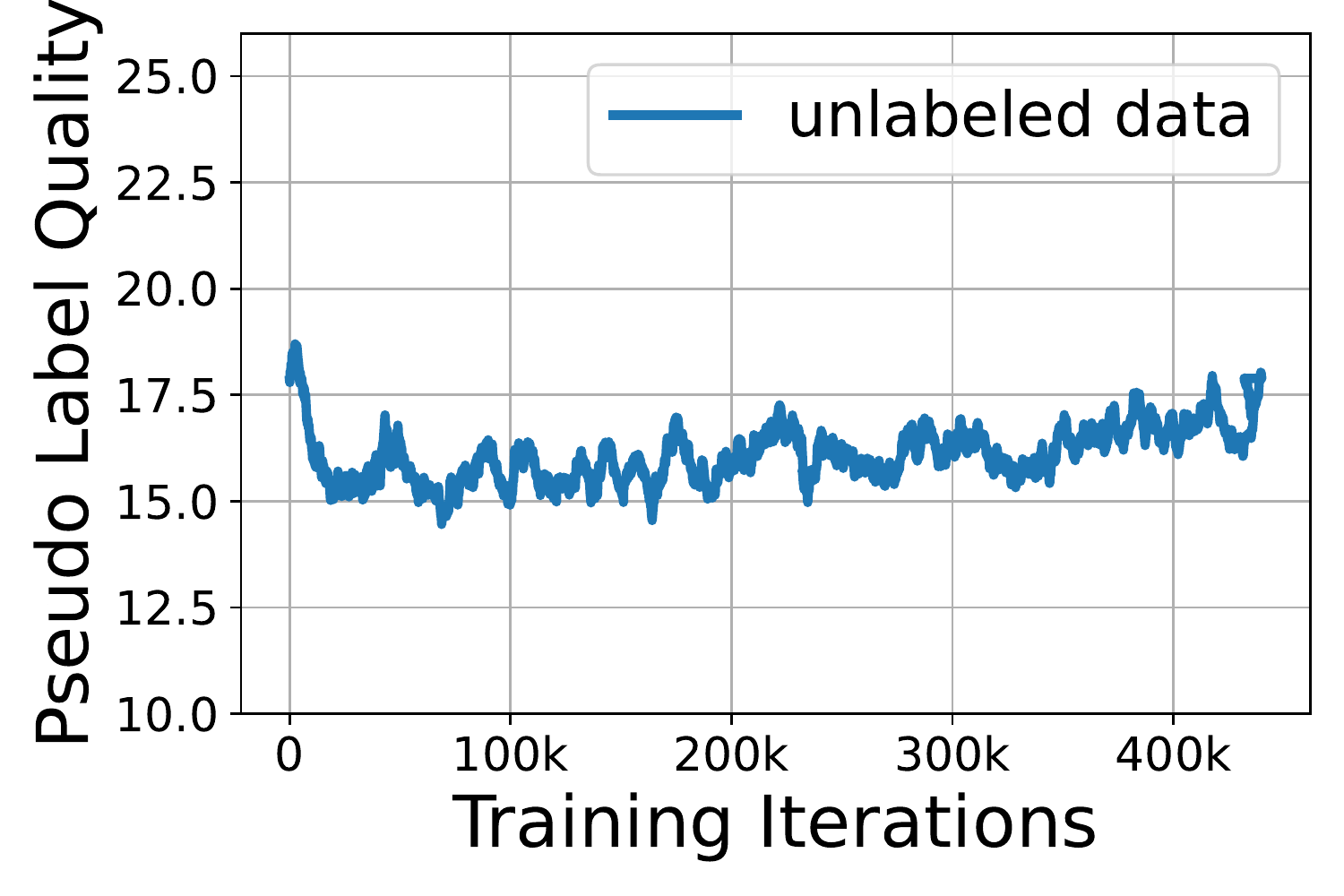} \hspace{-2mm}
			& \includegraphics[width=0.48\linewidth]{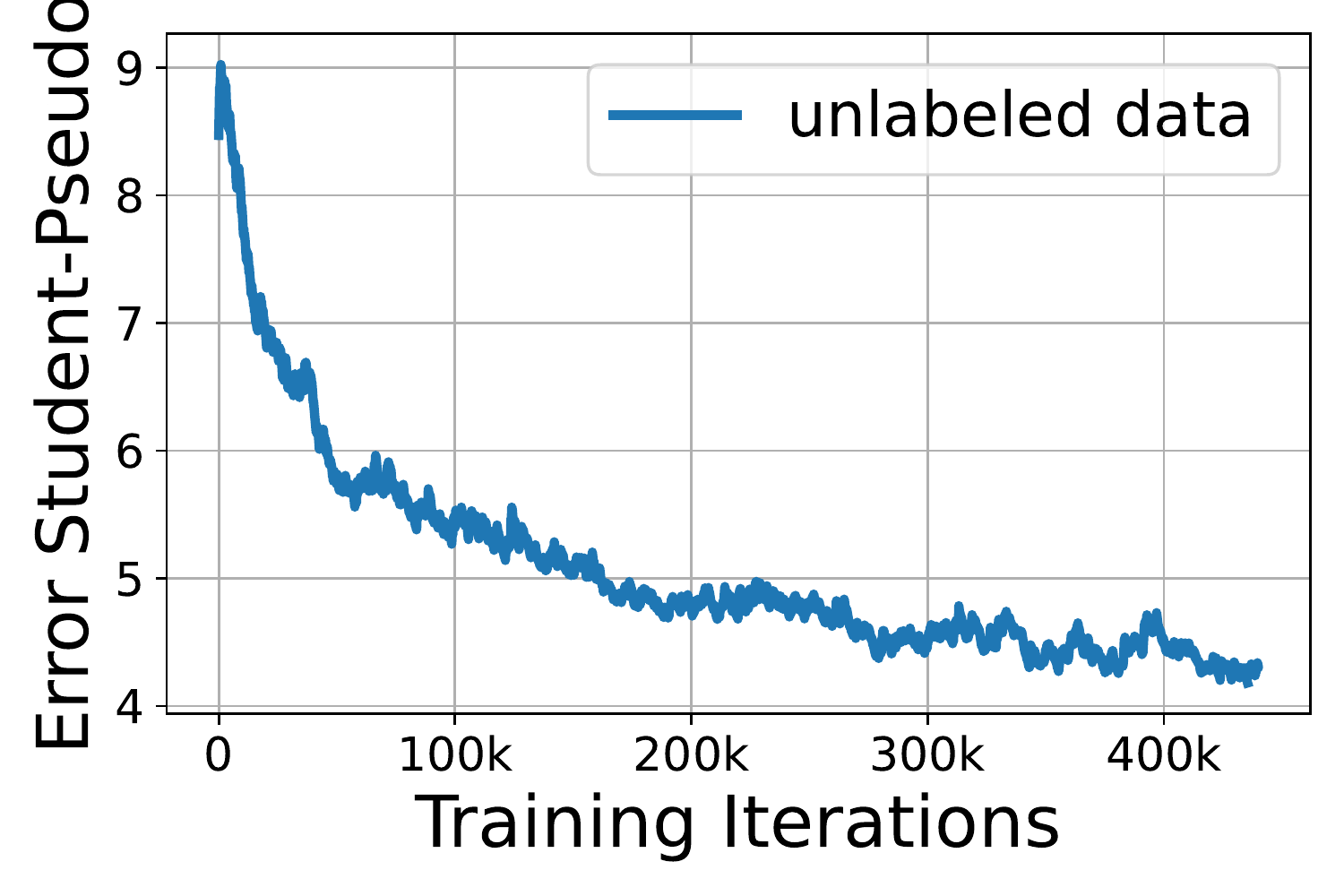}
		\end{tabular}
	\end{center}
	\vspace{-5mm}
	\caption{
	\textbf{The visualization of the training process} of SSL-FisherMatch on ModelNet10-SO(3) Sofa dataset with 5\% labeled data. 
	The four plots, from left to right and from top to bottom, show the mean errors of the predictions, the pseudo label coverage, the pseudo label quality represented by the mean errors of the pseudo labels, and the mean errors between the student model and the corresponding pseudo labels, in the process of training.
	All the errors are measured in degrees.
	}
	\label{fig:process}
	\vspace{-4mm}
\end{figure}

\paragraph{Result comparison}

Table \ref{tab:main} shows the results of our method compared with baselines on ModelNet10-SO(3)
under different labeled data ratios. 
We can see that the results of supervised learning with the labeled data only 
perform similarly, regardless of using a normal or a Fisher regressor. Since these models are in fact the pre-trained models for the SSL methods in the SSL stage, their similar performance sets a common basis for a fair comparison in the SSL stage.
For methods that undergo a second SSL stage, our proposed \textbf{FisherMatch} method consistently outperforms the baseline SSL method \textbf{SSL-L1-Consistency}, which demonstrates the importance of performing pseudo label filtering.

The experiment results on Pascal3D+ dataset are shown in Table \ref{tab:pascal}. The results illustrate that, with the effective teacher-student mutual learning framework as well as the entropy-based pseudo label filtering scheme, our algorithm significantly outperforms the state-of-the-art baselines under all different numbers of labeled images.

\paragraph{Training process analysis}
Here we show how our SSL method works during the training. In Fig. \ref{fig:process}, the upper left plot shows that the performance of the unlabeled data increases together with the test data, which indicates the increasing quality of the teacher predictions. We can also note that the performance on the unlabeled data is slightly better than that of the test data, which is sometimes referred to as \textit{transductive semi-supervised learning}. 

We also show the changes over the training process of the pseudo label coverage, the pseudo label quality, and the error between the student predictions and the corresponding pseudo labels, respectively. Here, we refer to \textit{pseudo labels} as the teacher predictions that pass the entropy threshold. The pseudo label coverage means the percentage of teacher predictions that pass the confidence threshold. The pseudo label quality simply means the error of the pseudo labels to the ground truth. 

As shown in the curves, as the SSL goes on, the improving model leads to more confident predictions indicated by the decreasing entropy and increasing pseudo label coverage, which in return fuels the learning process. The coverage of pseudo labels increases by a large margin from $40\%$ to the final $70\%$, while the pseudo label quality still keeps stable with a shaking around 2.5$^\circ$. This indicates that entropy always acts as a good indicator of performance during the whole process. The error of the student model to the pseudo labels keeps decreasing, which further proves the effectiveness of our unsupervised loss.

\subsection{Ablation Study}
\label{sec:analysis}

\paragraph{Effect of Different Unsupervised Loss and Entropy Threshold}
\begin{figure}[t]
\begin{center}
\includegraphics[width=0.9\linewidth]{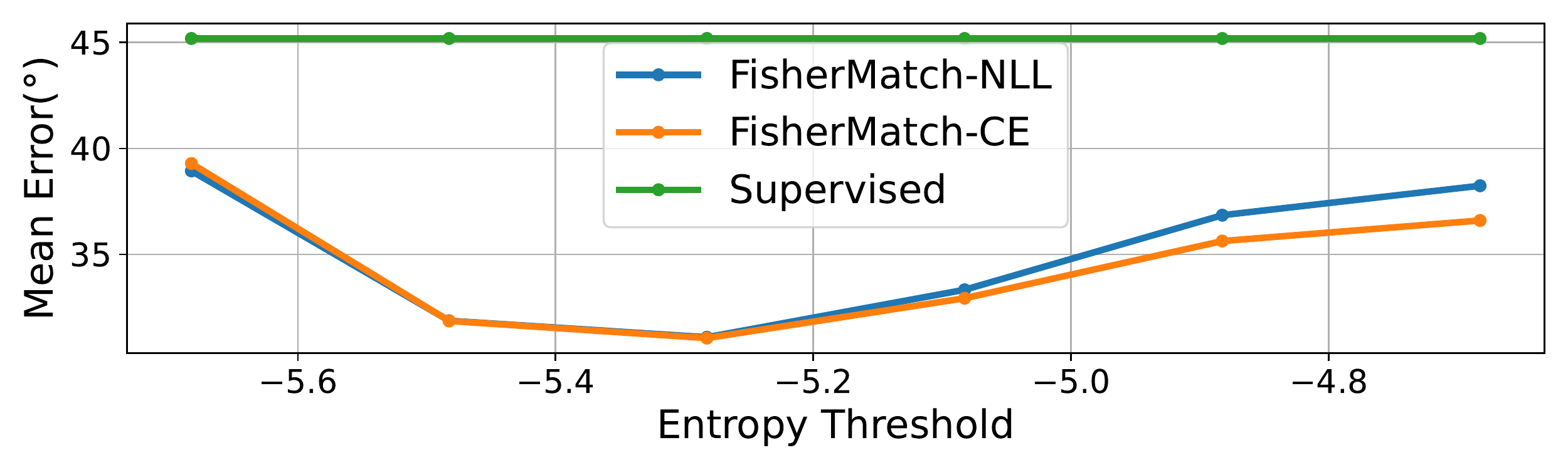}
\end{center}

\vspace{-5mm}
\caption{\textbf{The performance of FisherMatch with CE or NLL unsupervised losses with different entropy thresholds.} The experiments are done on ModelNet10-SO(3) Sofa dataset with 5\% labeled data.}
\vspace{-2mm}
\label{fig:ce_nll}
\end{figure}

\begin{figure}[t]
	\begin{center}
		\begin{tabular}{ccc}
			\hspace{-2mm}
			\includegraphics[width=0.31\linewidth]{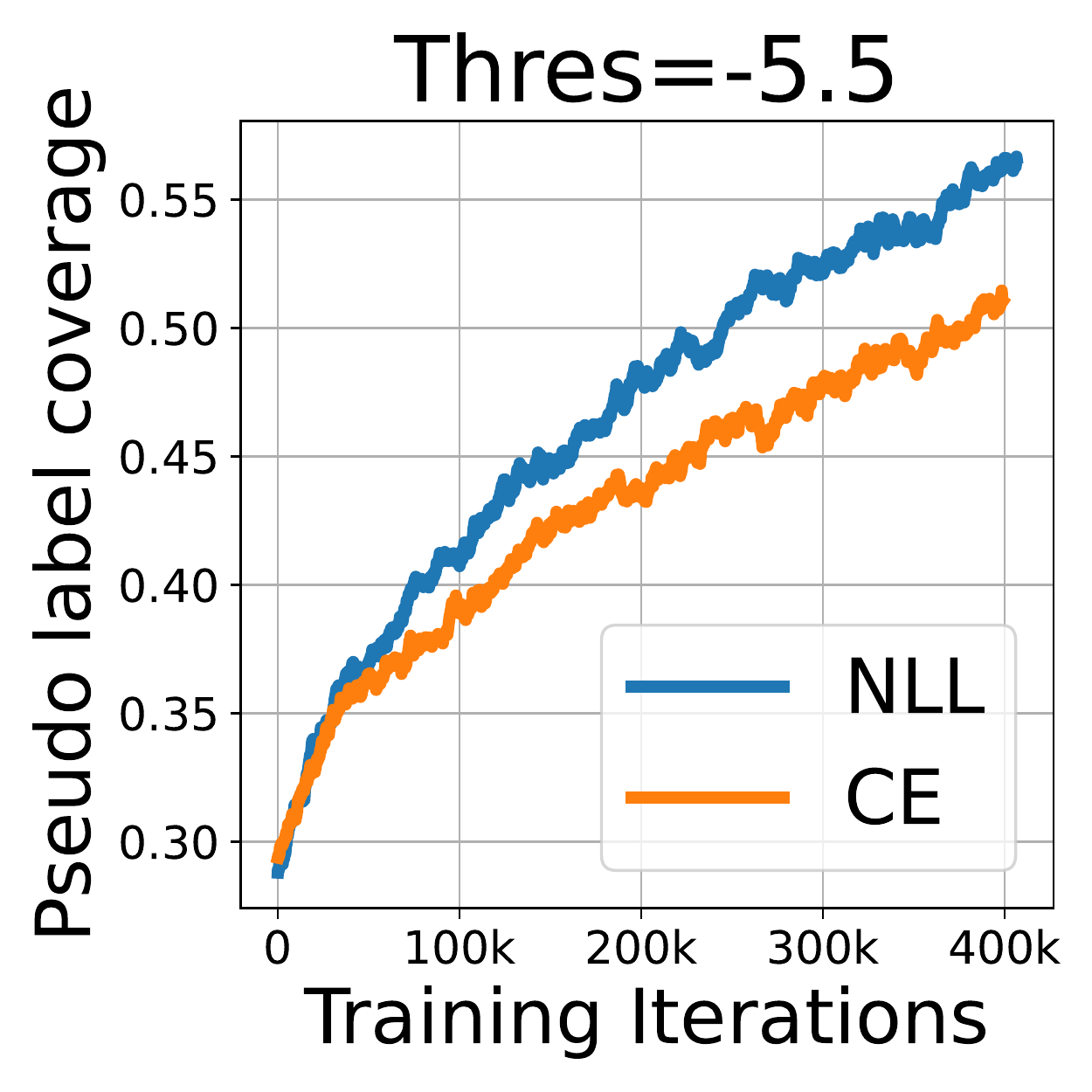} \hspace{-3mm}
			& \includegraphics[width=0.31\linewidth]{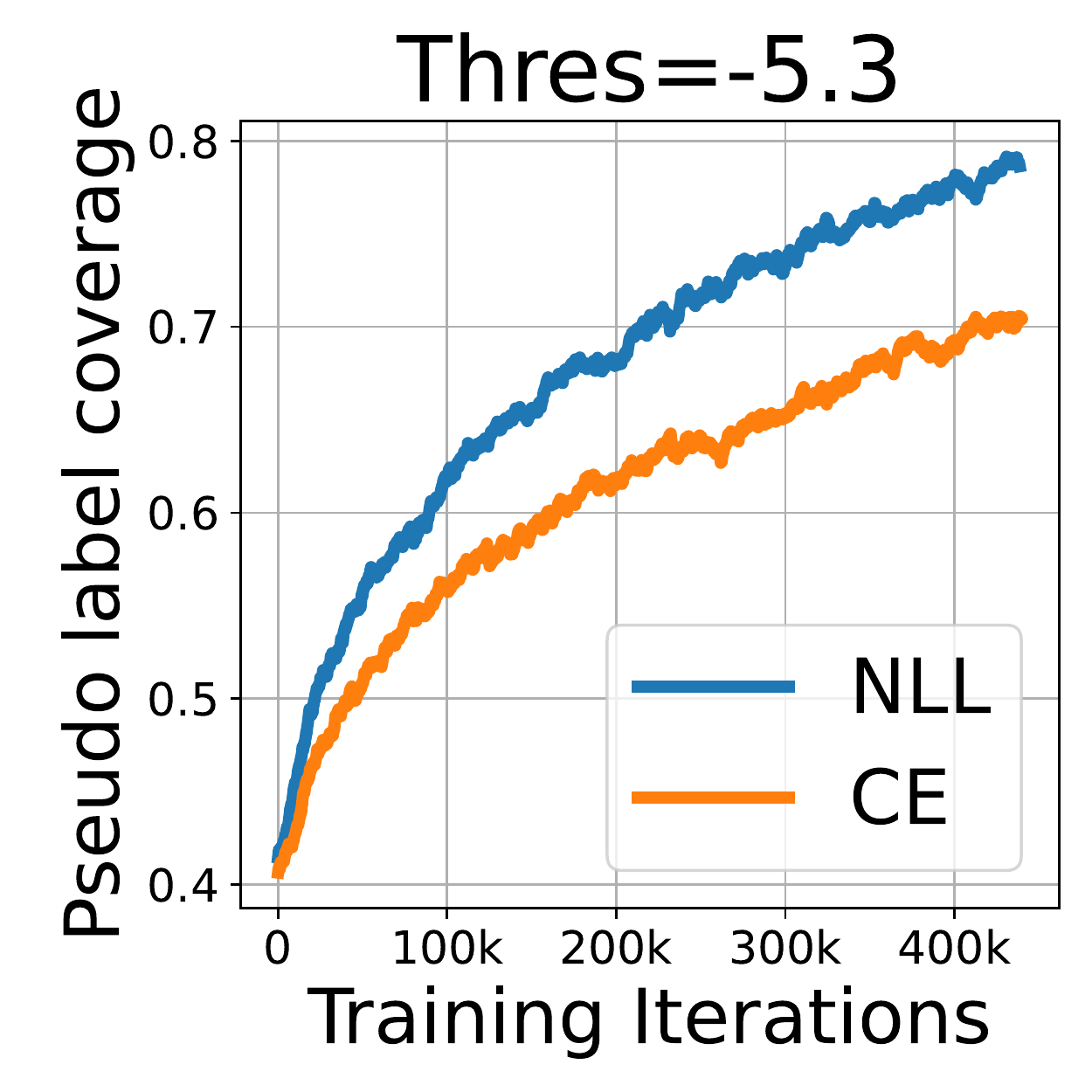} \hspace{-3mm}
			& \includegraphics[width=0.31\linewidth]{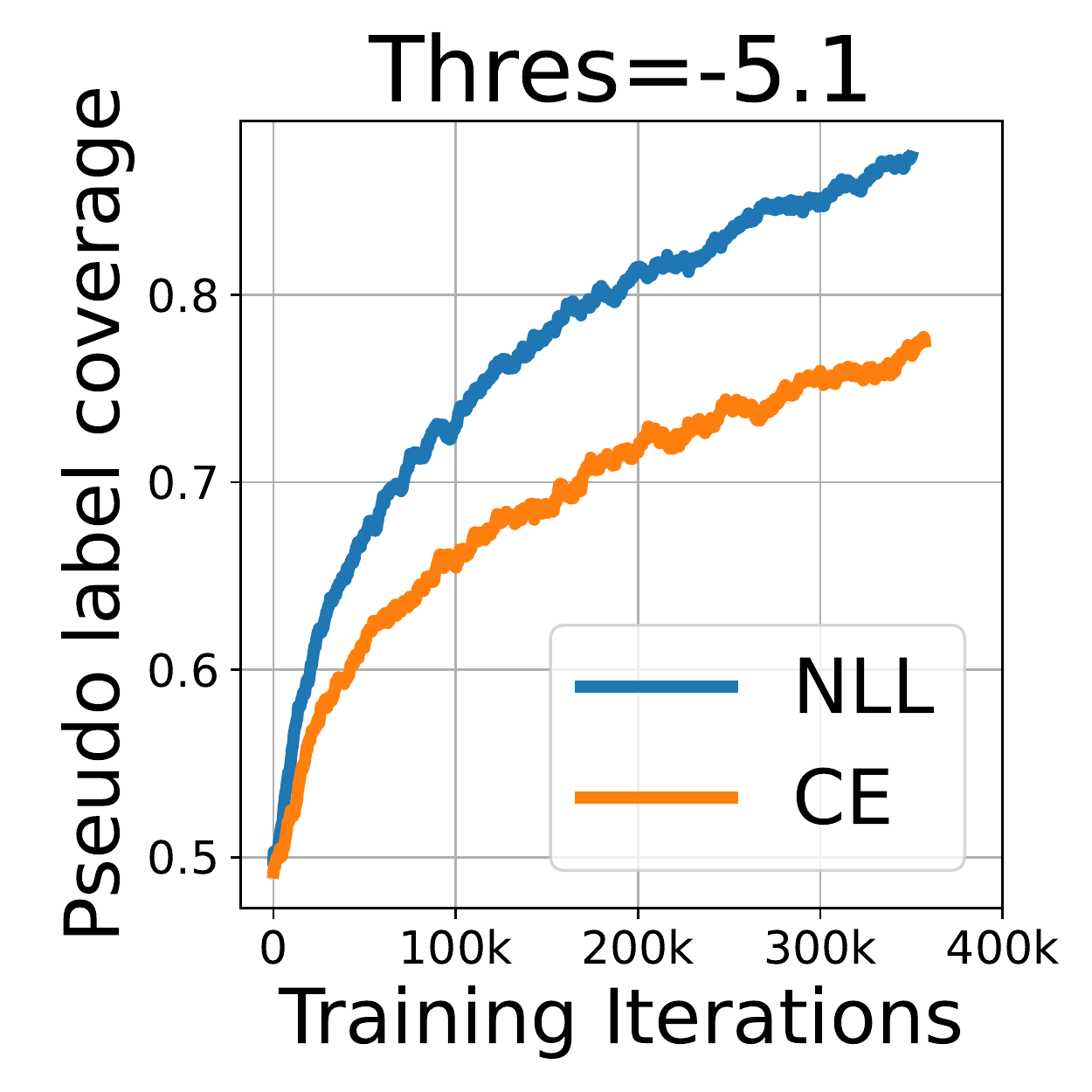} 
		\end{tabular}
	\end{center}
	\vspace{-7mm}
	\caption{\textbf{Comparison of the pseudo label coverage over training process with CE or NLL unsupervised losses and entropy thresholds.} The experiments are done on ModelNet10-SO(3) Sofa dataset with 5\% labeled data.}
	\label{fig:ce_nll_process}
\end{figure}

Here, we analyze the performance of our FisherMatch with different unsupervised losses, $L^{\text{CE}}$ (Eq. \ref{eq:ce_unsuper}) and $L^{\text{NLL}}$ (Eq. \ref{eq:nll_unsuper}), and how they are dependent on the entropy threshold $\tau$ by sweeping the parameter $\tau$.
Shown in Fig. \ref{fig:ce_nll} and \ref{fig:ce_nll_process}, the CE loss performs slightly better with a more tolerant threshold, while the NLL loss encourages a higher confidence of the network. The results verify that the NLL loss is a sharpened version of CE loss, where all the pseudo labels passing the threshold are seen as \textit{absolute confident} regardless of the actually predicted uncertainty. This behavior results in a more- but maybe over- confident network, especially with a tolerant entropy threshold. 
{On the other hand, since pseudo labels already exhibit much confidence as they pass the threshold, further sharpening does not lead to additional performance gains. }
Thus we believe CE loss is a better choice in our task with broader compatibility.

\vspace{-2mm}
\paragraph{Indication Ability of Distribution Entropy}

\begin{figure}[t]
    \centering
    \includegraphics[width=0.95\linewidth]{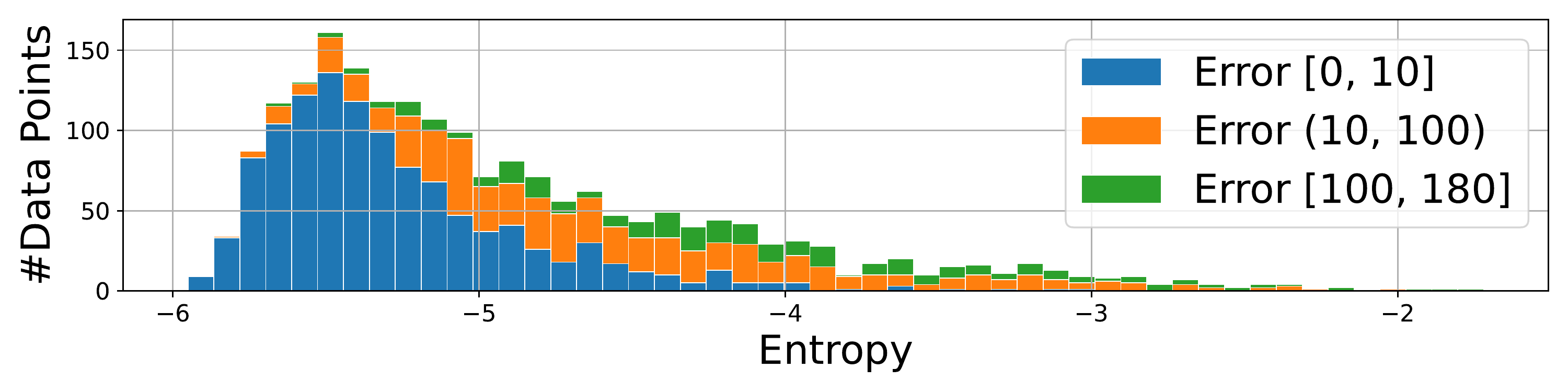}
    \vspace{-3mm}
    \caption{\textbf{Visualization of the indication ability of the distribution entropy wrt. the performance.} The horizontal axis is the distribution entropy and the vertical axis is the number of data points, color coded by the errors (in degrees). The experiments are done on ModelNet10-SO(3) Sofa dataset with 10\% labeled data.}
    \vspace{-2mm}
	\label{fig:indication}
\end{figure}

To clearly exhibit the indication ability of the distribution entropy wrt. the performance, we plot the relationship between the error of the prediction and the corresponding distribution entropy on test set in Fig. \ref{fig:indication}. The figure shows that the entropy effectively captures the prediction error, even under a low labeled data ratio.

\vspace{-2mm}
\paragraph{Comparison with Bingham-based Regressor}
\input{tex/table_bingham}

Our designed algorithm is agnostic to the choice of the rotation representation as well as the distribution model. We further test our framework based on the Bingham distribution and report the results in Table \ref{tab:bingham}.

As shown in the table, the Bingham-based framework is also able to utilize the unlabeled data and significantly improve the performance of rotation estimation. However, for both its supervised and semi-supervised version, its rotation errors are in general larger than those of matrix Fisher-based framework, since its rotation representation, quaternion, is not a continuous rotation representation, as pointed in \cite{zhou2019continuity}, thus leading to inferior performance. See Appendix Section \ref{sec:supp_settings} for detailed settings for SSL-BinghamMatch.

%% file: tex/table_main1.tex
{\renewcommand{\arraystretch}{0.95}
\begin{table*}[t]
  \footnotesize
  \setlength{\tabcolsep}{5mm}
  \centering
  \caption{ 
  \textbf{Comparing our proposed FisherMatch with the baselines on ModelNet10-SO(3) under different ratios of labeled data.} 
  }
    \begin{tabular}{c|l|cc|cc}
    \multicolumn{1}{c|}{\multirow{2}{*}{Category}} & \multicolumn{1}{c|}{\multirow{2}{*}{Method}} & \multicolumn{2}{c|}{5\%} & \multicolumn{2}{c}{10\%} \\
\cmidrule{3-6}        
&       & \multicolumn{1}{c}{Mean$\downarrow$} & \multicolumn{1}{c|}{Med.$\downarrow$} & \multicolumn{1}{c}{Mean$\downarrow$}   &\multicolumn{1}{c}{Med.$\downarrow$}  \\
    \midrule
    \multirow{4}[5]{*}{Sofa} & Sup.-L1 \cite{levinson2020analysis}       &   44.64   &   11.42  &    32.65  &    9.03   \\
& Sup.-Fisher \cite{mohlin2020probabilistic}         &   45.19   &   13.16  &   32.92  &    8.83      \\
& SSL-L1-Consist.        &   36.86   &   8.65 &     25.94  &    6.81        \\
& {SSL-FisherMatch}        &   \textbf{32.02}   &   \textbf{7.78}     & \textbf{21.29}  &    \textbf{5.25}       \\
\cmidrule{2-6} 
& Full Sup.   &    18.62  &    5.77 &  18.62  &    5.77 \\
    \midrule
    \multirow{4}[5]{*}{Chair} & Sup.-L1 \cite{levinson2020analysis}      &    40.41  &    16.09  &   29.02  &    10.64     \\
& Sup.-Fisher \cite{mohlin2020probabilistic}         &    39.34  &    16.79     &28.58  &    10.84    \\
& SSL-L1-Consist.        &    31.20  &    11.29    &23.59  &    8.10       \\
& {SSL-FisherMatch}        &    \textbf{26.69}  &    \textbf{9.42}       &\textbf{20.06}  &    \textbf{7.44}       \\
\cmidrule{2-6} 
& Full Sup. &    17.38  &    6.78 &    17.38  &    6.78 \\
\bottomrule
    \end{tabular}
  \label{tab:main}
\end{table*}
}

%% file: tex/table_pascal.tex
{\renewcommand{\arraystretch}{1.}
\begin{table*}[t]
  \centering
  \footnotesize 
  \setlength{\tabcolsep}{4mm}
  
  \caption{\textbf{Comparing our proposed FisherMatch with the baselines on the 6 categories of Pascal3D+ dataset with few annotations (7, 20, 50 images).} The results are averaged on 6 categories.}
  \vspace{-1mm}
    \begin{tabular}{l|cc|cc|cc}
    \multicolumn{1}{c|}{\multirow{2}{*}{Method}} & \multicolumn{2}{c|}{7} & \multicolumn{2}{c|}{20} & \multicolumn{2}{c}{50} \\
    \cmidrule{2-7}
    & Med.$\downarrow$ &Acc$_{30^\circ}$$\uparrow$    & Med.$\downarrow$ & Acc$_{30^\circ}$$\uparrow$ &  Med.$\downarrow$ &Acc$_{30^\circ}$$\uparrow$ \\
    \midrule
    Res50-Gene & 39.1  & 36.1  & 26.3  & 45.2  & 20.2  & 54.6 \\
    Res50-Spec & 46.5  & 29.6  & 29.4  & 42.8  & 23.0    & 50.4 \\
    StarMap\cite{zhou2018starmap} & 49.6  & 30.7  & 46.4  & 35.6  & 27.9  & 53.8 \\
    NeMo\cite{wang2021nemo}  & 60.0    & 38.4  & 33.3  & 51.7  & 22.1  & 69.3 \\
    NVSM\cite{wang2021neural}  & 37.5  & 53.8  & 28.7  & 61.7  & 24.2  & 65.6 \\
    {FisherMatch} & \textbf{28.3}  & \textbf{56.8}  & \textbf{23.8}  & \textbf{63.6}  & \textbf{16.1}  & \textbf{75.7} \\
    \midrule
    Full Sup. &8.1 & 89.6  &8.1 & 89.6 &8.1 & 89.6  \\
    \bottomrule
    \end{tabular}
  \label{tab:pascal}
\end{table*}
}

%% file: tex/table_bingham.tex
\begin{table}[t]
\setlength{\tabcolsep}{16pt}
\footnotesize
  \centering
  \caption{\textbf{Semi-supervised learning experiment based on the Bingham distribution} on ModelNet10-SO(3) Sofa dataset with 10\% labeled data. }
  \vspace{-1mm}
    \begin{tabular}{lcc}
     \multicolumn{1}{c}{Method}     & \multicolumn{1}{l}{Mean$\downarrow$} & \multicolumn{1}{l}{Med.$\downarrow$} \\
    \midrule
    Sup.-Bingham &    39.61   &  12.68 \\
    Sup.-Fisher &     32.92  &   8.83 \\
    SSL-BinghamMatch &    27.01   &   6.77 \\ 
    SSL-FisherMatch &   21.29     &  5.25  \\
    \bottomrule
    \end{tabular}%
    \vspace{-2mm}
  \label{tab:bingham}%
\end{table}%

%% file: tex/5_conclusion.tex
\section{Conclusion and {Limitations}}
\label{sec:conclusion}

In this paper, we tackle the problem of semi-supervised rotation regression from single RGB images in a general way. Without requiring any domain-specific knowledge or paired images, we leverage the teacher-student mutual learning framework and propose an entropy-based pseudo label filtering strategy based on the probabilistic modeling of $\SO$. Our experiments demonstrate the effectiveness and advantage of our method on both ModelNet10-SO(3) and Pascal3D+ datasets.

The performance of our method may degrade when both the numbers of labeled and unlabeled data are not sufficient. 
In this case, the uncertainty predicted by our network can be under-estimated due to over-fitting in the small labeled data, leading to reduced effectiveness in the pseudo label filtering and thus the mutual learning.

\section*{Acknowledgements}
We thank the anonymous reviewers for the insightful feedback.
We would like to credit Jiangran Lv from DUT for the fruitful discussions and valuable help in experiments 
and Yang Wang from PKU for the help in the derivation of maths. This work is supported in part by grants from the Joint NSFC-ISF Research Grant (62161146002).

%% file: supp.tex
\section{Implementation and Experiment Details}
\label{sec:supp_settings}

\subsection{Baselines in NVSM}
In the main paper, we compare our algorithm with NVSM\cite{wang2021neural} and their developed baselines, \textit{i.e.}, \textbf{StarMap}, \textbf{NeMo}, \textbf{Res50-Gene} and \textbf{Res50-Spec}. We briefly introduce these methods in this section, and more details can be found in \cite{wang2021neural}.

StarMap\cite{zhou2018starmap} and NeMo\cite{wang2021nemo} are two state-of-the-art supervised approaches for 3D pose estimation. For NeMo, the same single mesh cuboid is used as NVSM does. In addition, two baselines that formulate the object pose estimation problem as a classification task are adopted. To be specific, Res50-Gene formulates the pose estimation task for all categories as one single classification task, whereas Res50-Spec learns one classifier per category.

All baselines are evaluated using a semi-supervised protocol in a common pseudo labeling strategy. Specifically, all baselines are first trained on the annotated images and use the pretrained models to label the unlabeled data by pseudo labels. The final models are trained on both the annotated data and the pseudo-labeled data.

\subsection{Experiment Settings of BinghamMatch}
In Table 3 of the main paper, we experiment our algorithm based on the Bingham distribution $\mathcal{B}(\mathbf{M}, \mathbf{Z})$, namely BinghamMatch. We use the same experiment settings as FisherMatch, except that we choose unit quaternion as our rotation representation and use Bingham distribution for building the probabilistic rotation model. The rotation regressor outputs the parameters of the Bingham distribution. Specifically, following \cite{deng2020deep}, the regressor outputs a 7-d vector $(\mathbf{o}_1, \mathbf{o}_2)$ where the first 4-d vector $\mathbf{o}_1$ are first normalized and used to construct the parameter $\mathbf{M}$ via \textit{Birdal Strategy} 
\begin{equation*}
\mathbf{M}(\mathbf{o}_1) \triangleq\left[\begin{array}{rrrr}
o_{11} & -o_{12} & -o_{13} & o_{14} \\
o_{12} & o_{11} & o_{14} & o_{13} \\
o_{13} & -o_{14} & o_{11} & -o_{12} \\
o_{14} & o_{13} & -o_{12} & -o_{11}
\end{array}\right]
\end{equation*}
and the last 3-d vector $\mathbf{o}_2$ are applied by softplus activation and accumulation sum to construct the parameter $\mathbf{Z}$, with
\begin{equation*}
\begin{aligned}
&z_{1}=-\phi\left(o_{21}\right) \\
&z_{2}=-\phi\left(o_{21}\right)-\phi\left(o_{22}\right) \\
&z_{3}=-\phi\left(o_{21}\right)-\phi\left(o_{22}\right)-\phi\left(o_{23}\right)
\end{aligned}
\end{equation*}
where $\phi(\cdot)$ is the softplus activation.

\subsection{Implementation Details}
We run all the experiments with the unsupervised loss weight $\lambda_u$ as 1. In the pre-training stage, we train with the batch size of 32, and for the SSL stage, a training batch is composed of 32 labeled samples and 128 unlabeled samples. Both the weak and strong augmentations consist of random padding, cropping, resizing and color jittering (for real-world images) operations with different strengths.
On ModelNet10-SO(3) dataset, we use MobileNet-V2 \cite{howard2017mobilenets} architecture following \cite{levinson2020analysis, chen2021projective}. We use the Adam optimizer with the learning rate as 1e-4 without decaying. The entropy threshold $\tau$ is set as around -5.3. On Pascal3D+ dataset, we follow NVSM \cite{wang2021neural} to use ResNet \cite{he2016deep} architecture pretrained on ImageNet \cite{deng2009imagenet} dataset. We use the Adam optimizer with the learning rate as 1e-4 in pre-training stage and 1e-5 in the Semi-supervised training stage, without decaying. 
Due to the extremely small amount of data, we find a large variation among experiments of different categories and \#labeled images on Pascal3D+ dataset, thus choose different confidence thresholds in the SSL stage.

\section{Review of Bingham Distribution and Matrix Fisher Distribution}
\label{sec:supp_math}

\input{tex_supp/math_bingham}
\input{tex_supp/math_fisher}

\input{tex_supp/math_equivalence}

\input{tex_supp/math_constant1}

\section{More Experiment Results}
\label{sec:supp_results}

\subsection{Results on ModelNet10-SO(3) Dataset with 100\% Labeled Data}

Although out of the scope of semi-supervised learning, following \cite{wang20213dioumatch, mariotti2020semi}, we also report the results on 100\% labeled data on ModelNet10-SO(3) dataset, where we simply make a copy of the full training data as unlabeled data and train our model. All the other settings are kept the same as Table 1 in the main paper.

As shown in Table \ref{tab:100}, our proposed FisherMatch is able to further encourage a better performance with 100\% labeled data compared with the supervised learning and consistently outperforms other baselines. The results further demonstrate the importance of filtering high-quality pseudo labels even with much training data. The improvements can be seen as a result of label smoothing\cite{yuan2020revisiting}.

\subsection{Experiments and Results on Objectron Dataset}
\noindent\textbf{Dataset} \ Objectron \cite{ahmadyan2021objectron} is a newly-introduced dataset captured in the real world. The dataset contains a collection of short, object-centric video clips, as well as the corresponding camera poses, sparse point clouds, and manually annotated 3D bounding boxes for each object.

In this experiment, we mainly focus on the  \texttt{bike} and \texttt{camera} categories which exhibit more rotational variations and less rotational symmetries in the dataset \cite{ahmadyan2021objectron}. 
Since the real-world images are mostly captured from limited viewpoints, we found a smaller generalization gap between the train/test data. Thus, we choose a more challenging scenario to only adopt 1\% labeled data to train the network. 
We adopt the official train-test split of the dataset, where we grab all the frames of the training videos and uniformly sample 10\% frames from the test videos. We further divide the training split into the labeled set with ground truth and the unlabeled set without ground truth.

\input{tex_supp/table_100}
\input{tex_supp/table_objectron}

\noindent\textbf{Data preprocessing} \ To leverage this dataset for object pose regression, we need to obtain the paired data, \textit{i.e.}, object-centered images with their corresponding object poses.
We thus first project the eight corners of 3D bounding box annotations onto the 2D image plane, fit a minimum 2D square bounding box covering all the projected corners, and finally crop the image with the fitted 2D bounding box. 
To avoid the naive cropping-resizing flaws pointed out in \cite{mohlin2020probabilistic}, we directly crop square images to meet the shape requirement of the network. We pad the images with a black background to cover the out-of-plane projected keypoints and images with more than 4 (out of 8) keypoints out of the image plane are discarded.
To obtain the ground-truth object poses, we compute the rotation of the annotated 3D object bounding box wrt. the box with the same size in the canonical orientation.

\noindent\textbf{Experiment settings} \ The baselines, evaluation metrics and implementation details are the same as experiments on ModelNet10-SO(3) dataset.

\noindent\textbf{Results} \ The results are shown in Table \ref{tab:objectron}. Our FisherMatch significantly increases the regression performance even with a really low labeled data ratio, further demonstrating the efficiency of our model.

\section{Visualization of Matrix Fisher Distribution}
\label{sec:supp_vis}

\begin{figure}[t]
    \centering
    \begin{subfigure}[t]{0.3\linewidth}
        \centering
        \includegraphics[width=\linewidth]{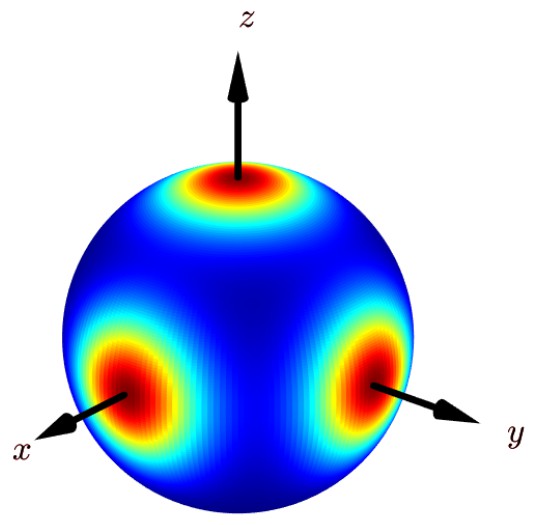}
        \caption{diag(5, 5, 5)}
        \label{indication_entropy}
    \end{subfigure}
    \hspace{2mm}
    \begin{subfigure}[t]{0.3\linewidth}
        \centering
        \includegraphics[width=\linewidth]{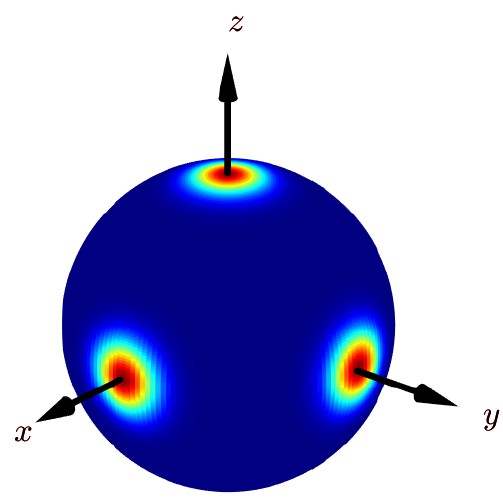}
        \caption{diag(20, 20, 20)}
    \end{subfigure}
    \hspace{2mm}
    \begin{subfigure}[t]{0.3\linewidth}
        \centering
        \includegraphics[width=\linewidth]{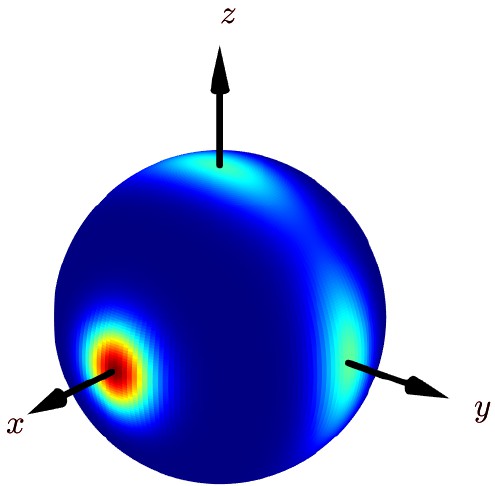}
        \caption{diag(20, 1, 1)}
    \end{subfigure}
	\caption{\textbf{Visualization of the pdf of matrix Fisher distribution} with \textit{jet} color-coding. The captions below the plots indicate the parameter $\mathbf{A}$ of the distribution.}
	\label{fig:vis}
\end{figure}

\begin{figure}[t]
	\begin{center}
		\begin{tabular}{cccc}
		\hspace{-4mm} \includegraphics[width=0.24\linewidth]{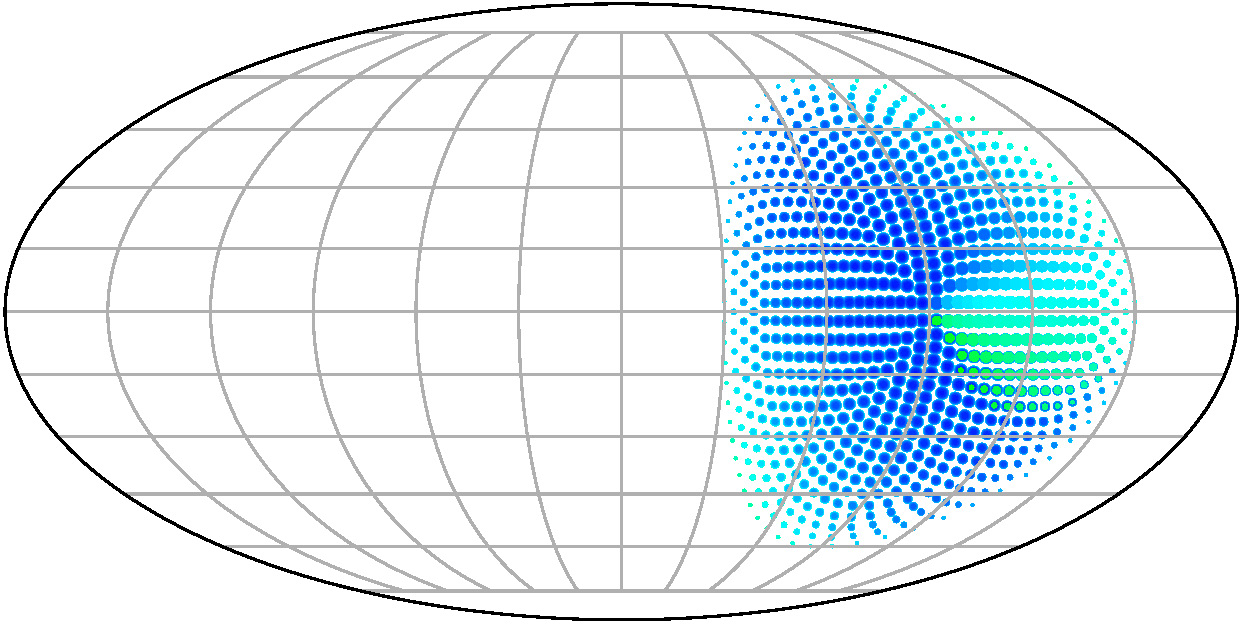} \hspace{-6mm}
			& \includegraphics[width=0.24\linewidth]{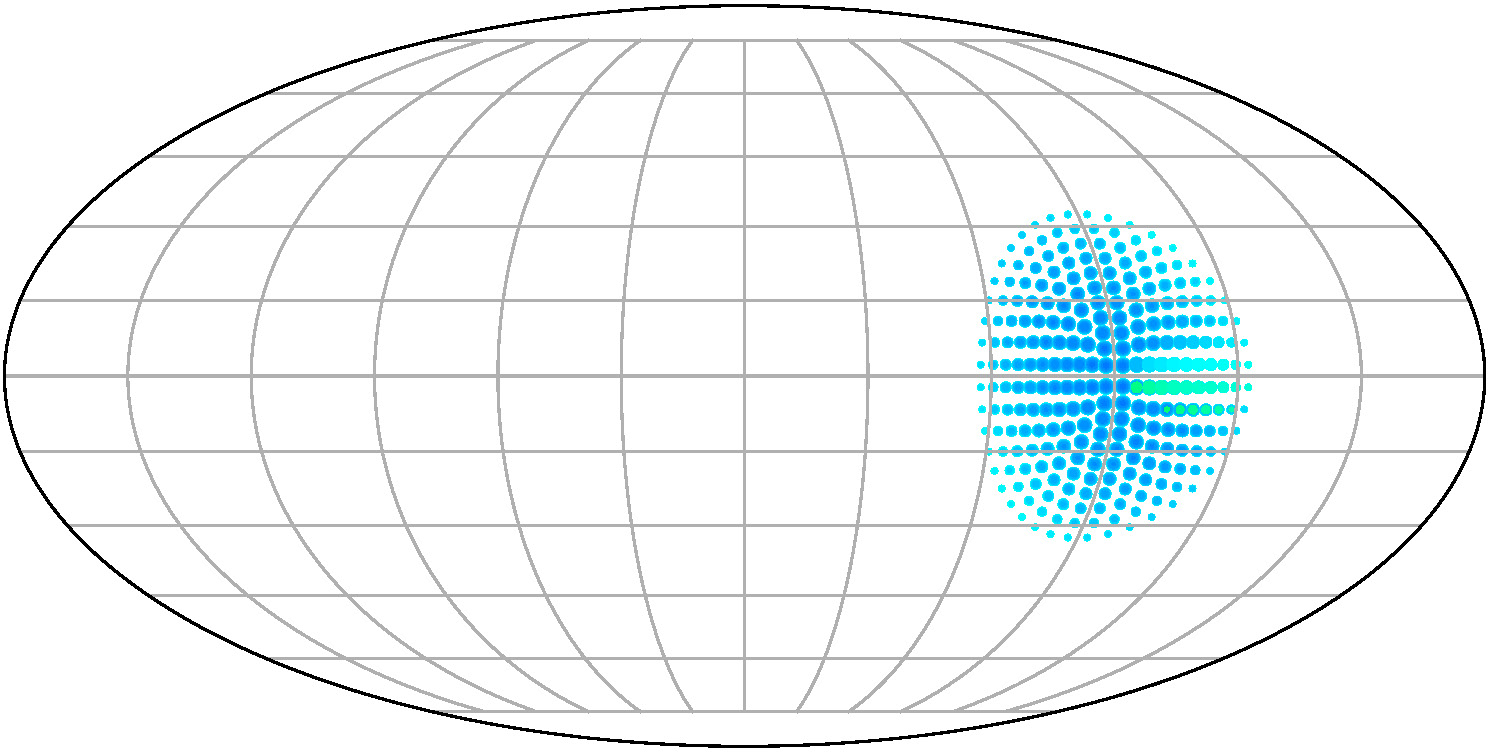}    \hspace{-6mm}
			& \includegraphics[width=0.24\linewidth]{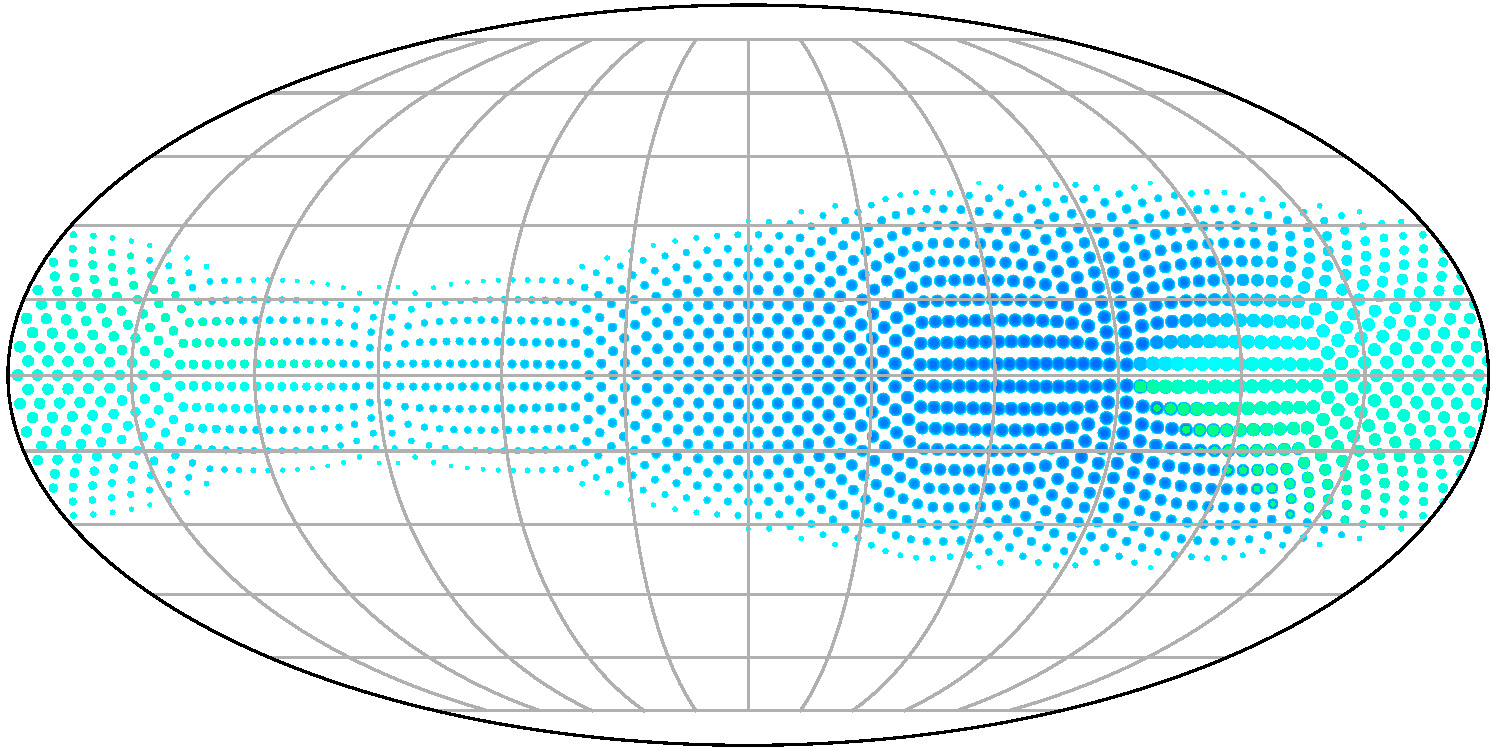} \hspace{-6mm}
			& \includegraphics[width=0.14\linewidth]{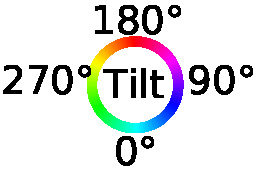} \\
		\hspace{-6mm} \scriptsize{(a) diag(5, 5, 5)}  \hspace{-6mm} 
		& \scriptsize{(b) diag(20, 20, 20)}  \hspace{-6mm}  
		& \scriptsize{(c) diag(20, 1, 1)} \hspace{-6mm}  & 
		\end{tabular}
	\end{center}
	\vspace{-4mm}
	\caption{\textbf{Visualization of the pdf of matrix Fisher distribution} with the visualization method proposed in Implicit-PDF \cite{murphy2021implicit}. The captions below the plots indicate the parameter $\mathbf{A}$ of the distribution.}
	\label{fig:vis1}
\end{figure}

Visualizing matrix Fisher distribution is non-trivial over SO(3). Following \cite{mohlin2020probabilistic,lee2018bayesian}, we visualize the probabilistic distribution function via color-coding on the sphere. 

Remember that for the parameter $\mathbf{A}$ in matrix Fisher distribution, the singular values indicate the strength of concentration. The larger a singular value $s_i$ is, the more concentrated the distribution is along the corresponding axis. Fig \ref{fig:vis} shows three distributions with the same mode as the identity matrix, differing only in the strength of concentration.
For both (a) and (b), the distributions of each axis are identical and circular, while the distribution in (b) is more concentrated than (a). In (c), the distribution is more concentrated in $x$-axis, and the distributions for the other two axes are elongated.

Implicit-PDF \cite{murphy2021implicit} proposes a new visualization method to display distributions over $\SO$ by discretizing $\SO$ with the help of Hopf fibration\cite{yershova2010generating}. 
It projects a great circle of points on $\SO$ to each point on the 2-sphere and uses the color wheel to indicate the location on the great circle. We re-draw Figure \ref{fig:vis} with this visualization in Figure \ref{fig:vis1}.

%% file: tex_supp/math_bingham.tex
\subsection{Unit Quaternion and Rotation Matrix}
\label{sec:supp_basis}
Unit quaternion and rotation matrix are two commonly used representations for rotation elements from $\SO$. 
Unit quaternion $\mathbf{q} \in \mathcal{S}^3$ is a double-covered representation of $\SO$, where $\mathbf{q}$ and $-\mathbf{q}$ represent the same rotation.
Rotation $\mathbf{R} \in \mathbb{R}^{3\times 3}$ satisfies $\mathbf{R}^T\mathbf{R}=\mathbf{I}$ and $\text{det}(\mathbf{R}) = +1$.
For a quaternion $\mathbf{q} = [w,x,y,z]$, we use the standard transform function $\gamma$ to compute its corresponding rotation matrix:

{
\small
\begin{equation*}
    \gamma(\mathbf{q}) =
    \left[
    \begin{array}{ccc}
    1-2y^{2}-2 z^{2} & 2 x y-2 w z & 2 x z+2 w y \\ 
    2 x y+2 w z & 1-2 x^{2}-2 z^{2} & 2 y z-2 w x \\ 
    2 x z-2 w y & 2 y z+2 w x & 1-2 x^{2}-2 y^{2}
    \end{array}
    \right]
\end{equation*}
}
The inverse transform $\gamma^{-1}$ is \\
{
\small
\begin{equation*}
    \gamma^{-1}(\mathbf{R}) =
    \left[
    \begin{array}{c}
    \sqrt{1+\mathbf{R}_{00}+\mathbf{R}_{11}+\mathbf{R}_{22}}/2\\ 
    (\mathbf{R}_{21}-\mathbf{R}_{12})/2\sqrt{1+\mathbf{R}_{00}+\mathbf{R}_{11}+\mathbf{R}_{22}}\\ 
    (\mathbf{R}_{02}-\mathbf{R}_{20})/2\sqrt{1+\mathbf{R}_{00}+\mathbf{R}_{11}+\mathbf{R}_{22}}\\
    (\mathbf{R}_{10}-\mathbf{R}_{01})/2\sqrt{1+\mathbf{R}_{00}+\mathbf{R}_{11}+\mathbf{R}_{22}}
    \end{array}
    \right]
\end{equation*}
}
Note that we here only cover one hemisphere of $\mathcal{S}^3$.

\subsection{Bingham Distribution}

\textit{Bingham} distribution \cite{bingham1974antipodally, glover2014quaternion} is an antipodally symmetric distribution. Its probability density function $\mathcal{B}: \mathcal{S}^{d-1} \rightarrow \mathcal{R}$ is defined as

\begin{equation}
\label{eq:supp_bingham}
p_B(\mathbf{q})=
\mathcal{B}(\mathbf{q} ; \mathbf{M}, \mathbf{Z}) =\frac{1}{F(\mathbf{Z})} \exp \left(\mathbf{q}^{T} \mathbf{M} \mathbf{Z} \mathbf{M}^{T} \mathbf{q}\right)
\end{equation}
where $\mathbf{M} \in \text{O}(4)$ is a $4 \times 4$ orthogonal matrix and $\mathbf{Z} = \textrm{diag}(0, z_1, z_2, z_3)$ is a $4 \times 4$ diagonal matrix with $0\ge z_1 \ge z_2 \ge z_3$. The first column of parameter $\mathbf{M}$ indicates the mode and the remaining columns describe the orientation of dispersion while the corresponding $z_i, (i\in {1,2,3})$ describe the strength of the dispersion. $F(\mathbf{Z})$ is the normalizing constant.

\begin{proposition}
Given $f \sim \mathcal{B}(\mathbf{M}, \mathbf{Z})$, the \textbf{entropy} of Bingham distribution is computed as
\begin{equation}
H_B(f)=\log F - \mathbf{Z} \frac{\nabla F}{F}.
\end{equation}
\end{proposition}

\begin{proof}
Denote $C=\mathbf{M} \mathbf{Z} \mathbf{M}^{T}$
\begin{equation*}
\begin{aligned}
H_B(f) &=-\oint_{\mathbf{q} \in \mathcal{S}^{3}} f(\mathbf{q}) \log f(\mathbf{q})\mathrm{d}\mathbf{q} \\
&=-\oint_{\mathbf{q} \in \mathcal{S}^{3}} \frac{1}{F} \exp\left({\mathbf{q}^{T} C \mathbf{q}}\right)\left(\mathbf{q}^{T} C \mathbf{q}-\log F\right)\mathrm{d}\mathbf{q} \\
&=\log F-\frac{1}{F} \oint_{\mathbf{q} \in \mathcal{S}^{3}} \mathbf{q}^{T} C \mathbf{q} \exp{\left(\mathbf{q}^{T} C \mathbf{q}\right)}.
\end{aligned}
\end{equation*}
Writing $f$ in standard form, and denoting the hyperspherical integral by $g(\mathbf{Z})$,
\begin{equation*}
g(\mathbf{Z}) = \oint_{\mathbf{q} \in \mathcal{S}^{3}} \mathbf{q}^{T} C \mathbf{q} \exp{\left(\mathbf{q}^{T} C \mathbf{q}\right)}\mathrm{d}\mathbf{q},
\end{equation*}
Then 
\begin{equation*}
\begin{aligned}
    g(\mathbf{Z}) &=\oint_{\mathbf{q} \in \mathcal{S}^{3}} \sum_{i=1}^{4} z_{i}\left(\mathbf{v}_{i}^{T} \mathbf{q}\right)^{2} \exp{\left( \sum_{j=1}^{4} z_{j}\left(\mathbf{v}_{j}{ }^{T} \mathbf{q}\right)^{2}\right)}\mathrm{d}\mathbf{q} \\
    &=\sum_{i=1}^{4} z_{i} \frac{\partial F}{\partial z_{i}}=\mathbf{Z} \cdot \nabla \mathbf{F}.
\end{aligned}
\end{equation*}
Thus, the entropy is $\log F - \mathbf{Z} \frac{\nabla F}{F}$
\end{proof}

\begin{proposition}
Given $f \sim \mathcal{B}(\mathbf{M}_f, \mathbf{Z}_f)$ and $g \sim \mathcal{B}(\mathbf{M}_g, \mathbf{Z}_g)$, the \textbf{cross entropy} between Bingham distributions ($f$ to $g$) is computed as
\begin{equation}
H_B(f, g)=\log F_{g}-\sum_{i=1}^{4} z_{g i}\left(b_{i}^{2}+\sum_{j=1}^{4}\left(a_{i j}^{2}-b_{i}^{2}\right) \frac{1}{F_{f}} \frac{\partial F_{f}}{\partial z_{f j}}\right).
\end{equation}
where $a_{ij}$ is the entries of $\mathbf{\hat{A}} = \mathbf{M}_{f}^{T} \mathbf{M}_{g}$ and $b_{i}$ is the entries of  $\mathbf{b}=\boldsymbol{\mu}_{\mathbf{f}}^{T} \mathbf{M}_{g}$ ($\boldsymbol{\mu}_{f}$ is the mode of distribution $f$).
\end{proposition}

\begin{proof}
\begin{equation*}
\begin{aligned}
H_B(f, g) &=-\oint_{\mathbf{q} \in \mathcal{S}^{3}} f(\mathbf{q}) \log g(\mathbf{q})\mathrm{d}\mathbf{q} \\
&=-\oint_{\mathbf{q} \in \mathcal{S}^{3}} f(\mathbf{q}) \left(\sum_{i=1}^{4} z_{g i}\left(\mathbf{v}_{\mathrm{gi}}^{T} \mathbf{q}\right)-\log F_{g}\right)\mathrm{d}\mathbf{q} \\
&=\log F_{g}-\sum_{i=1}^{4} z_{g i} E_{f}\left[\left(\mathbf{v}_{\mathrm{gi}}^{T} \mathbf{q}\right)\right].
\end{aligned}
\end{equation*}
Since $\left[\begin{array}{c}\mathbf{A}\\ \mathbf{b}^{T}\end{array}\right]=\left[\begin{array}{c}\mathbf{M}_{f}^{T} \\ \boldsymbol{\mu}_{f}^{T}\end{array}\right] \mathbf{M}_{g}$ 
and $\left[\begin{array}{c}\mathbf{M}_{f}^{T} \\ \boldsymbol{\mu}_{f}^{T}\end{array}\right]$ is orthogonal, $\mathbf{M}_{g}=\left[\mathbf{M}_{f} \boldsymbol{\mu}_{f}\right]\left[\begin{array}{c}\mathbf{A} \\ \mathbf{b}^{T}\end{array}\right]$, so $\mathbf{v}_{g i}= \mathbf{M}_{f} \mathbf{a}_{i}+b_{i} \boldsymbol{\mu}_{f}$. Thus,

\begin{equation*}
\begin{aligned}
E_{f}\left[\mathbf{v}_{gi}^{T} \mathbf{q}\right] &=E_{f}\left[\left(\left(\mathbf{M}_{f} \mathbf{a}_{i}+b_{i} \boldsymbol{\mu}_{f}\right)^{T} \mathbf{q}\right)^{2}\right] \\
&=b_{i}^{2} E_{f}\left[\left(\boldsymbol{\mu}_{f}^{T} \mathbf{q}\right)^{2}\right]+\sum_{j=1}^{4} a_{i j}^{2} E_{f}\left[\left(\mathbf{v}_{fj}^{T} \mathbf{q}\right)^{2}\right]
\end{aligned}
\end{equation*}
by linearity of expectation, and since all the odd projected moments are zero. Since
\begin{equation*}
    E_{f}\left[\left(\boldsymbol{\mu}_{f}^{T} \mathbf{q}\right)^{2}\right]=1-\sum_{j=1}^{4} E_{f}\left[\left(\mathbf{v}_{fj}^{T} \mathbf{q}\right)^{2}\right]
\end{equation*}
and
\begin{equation*}
    E_{f}\left[\left(\mathbf{v}_{fj}^{T} \mathbf{q}\right)^{2}\right]=\frac{1}{F_{f}} \frac{\partial F_{f}}{\partial z_{f j}},
\end{equation*}
then 
\begin{equation*}
H(f, g)=\log F_{g}-\sum_{i=1}^{4} z_{g i}\left(b_{i}^{2}+\sum_{j=1}^{4}\left(a_{i j}^{2}-b_{i}^{2}\right) \frac{1}{F_{f}} \frac{\partial F_{f}}{\partial z_{f j}}\right).
\end{equation*}
\end{proof}

%% file: tex_supp/math_fisher.tex
\subsection{Matrix Fisher Distribution}

\textit{Matrix Fisher} distribution \cite{prentice1986orientation,khatri1977mises} $\mathcal{MF}(\mathbf{R};\mathbf{A})$ is a probability distribution over SO(3) for rotation matrices, whose probability density function is in the form of 
\begin{equation}
\label{eq:supp_fisher}
p_F(\mathbf{R})=\mathcal{{MF}}(\mathbf{R} ; \mathbf{A})=\frac{1}{F(\mathbf{A})} \exp \left(\operatorname{tr}\left(\mathbf{A}^{T} \mathbf{R}\right)\right)
\end{equation}
where parameter $\mathbf{A} \in \mathbb{R}^{3\times 3}$ is an arbitrary $3\times 3$ matrix and $F(\mathbf{A})$ is the normalizing constant. The mode and dispersion of the distribution can be computed from the singular value decomposition of the parameter $\mathbf{A}$. Assume $\mathbf{A} = \mathbf{USV}^T$ and the singular values are sorted in descending order, the mode of the distribution is computed as 
\begin{equation*}
\mathbf{\hat{R}}=\mathbf{U}\left[\begin{array}{ccc}
1 & 0 & 0 \\
0 & 1 & 0 \\
0 & 0 & \operatorname{det}(\mathbf{U} \mathbf{V})
\end{array}\right] \mathbf{V}^{T}
\end{equation*}
and the singular values $\mathbf{S} = \text{diag}(s_1, s_2, s_3)$ indicates the strength of concentration. The larger a singular value $s_i$ is, the more concentrated the distribution is along the corresponding axis (the $i$-th column of mode $\mathbf{\hat{R}}$).

\paragraph{Entropy and Cross Entropy} 
Given $f\sim\mathcal{MF}(\mathbf{A}_f)$ and $g\sim\mathcal{MF}(\mathbf{A}_g)$, we can start with the definition,
\begin{equation*}
    H_F(f) = -\int_{\mathbf{R} \in \SO} f(\mathbf{R}) \log f(\mathbf{R})\mathrm{d}\mathbf{R}
\end{equation*}
and
\begin{equation*}
    H_F(f, g) = -\int_{\mathbf{R} \in \SO} f(\mathbf{R}) \log g(\mathbf{R})\mathrm{d}\mathbf{R}.
\end{equation*}
However, note the equivalence of matrix Fisher distribution and Bingham distribution (see Section \ref{sec:supp_equ}), and doing integrals over $\mathbb{S}^3$ (with 4 dimensions and 1 constraint) is easier than that over SO(3) (with 9 dimensions and 6 constraints), we first convert a matrix Fisher distribution to its equivalent Bingham distribution, and compute the properties via the formula of Bingham distribution.

Let $p_F$ be the pdf of a matrix Fisher distribution, and $p_B$ be the pdf of its equivalent Bingham distribution. Based on Eq. \ref{eq:equ_d_2pi2} and \ref{eq:equ_p_2pi2} in Section \ref{sec:supp_equ}, we have
\begin{equation}
    \begin{aligned}
        H_F(p_F) &= -\int_{\mathbf{R} \in \SO} p_F \log p_F \mathrm{d}\mathbf{R} \\
        &= -\oint_{\mathbf{q} \in \mathbb{S}^{3}} 2\pi^2 p_B \left(\log(2\pi^2) + \log (p_B)\right) \frac{1}{2\pi^2}\mathrm{d}\mathbf{q} \\
        &= -\log (2\pi^2)\oint_{\mathbf{q} \in \mathbb{S}^{3}} p_B \mathrm{d}\mathbf{q}  - \oint_{\mathbf{q} \in \mathbb{S}^{3}} p_B\log \mathrm{d}\mathbf{q} \\
        &= H_B(p_B) - \log (2\pi^2).
    \end{aligned}
\end{equation}
And similarly,
\begin{equation}
    H_F(f,g) = H_B(f,g) - \log (2\pi^2).
\end{equation}

%% file: tex_supp/math_equivalence.tex
\subsection{Equivalence of Bingham Distribution and Matrix Fisher Distribution}
\label{sec:supp_equ}

\begin{figure*}[ht]
\begin{equation}
\label{eq:supp_proper}
\mathbf{A}=\mathbf{U}_{1} \mathbf{S}^{\prime} \mathbf{V}_{1}^{T}=\underbrace{\mathbf{U}_{1}\left[\begin{array}{ccc}
1 & 0 & 0 \\
0 & 1 & 0 \\
0 & 0 & \operatorname{det}\left(\mathbf{U}_{1}\right)
\end{array}\right]}_{\mathbf{U}} \underbrace{\left[\begin{array}{ccc}
s_{1}^{\prime} & 0 & 0 \\
0 & s_{2}^{\prime} & 0 \\
0 & 0 & \operatorname{det}\left(\mathbf{U}_{1} \mathbf{V}_{1}\right) s_{3}^{\prime}
\end{array}\right]}_{\mathbf{S}} \underbrace{\left[\begin{array}{ccc}
1 & 0 & 0 \\
0 & 1 & 0 \\
0 & 0 & \operatorname{det}\left(\mathbf{V}_{1}\right)
\end{array}\right]}_{\mathbf{V}^{T}} \mathbf{V}_{1}^{T}=\mathbf{U} \mathbf{S} \mathbf{V}^{T}
\end{equation}
\end{figure*}

As discussed in \cite{prentice1986orientation}, for a random rotation matrix variable $\mathbf{R}$, it follows a matrix Fisher distribution if and only if its corresponding unit quaternion $\mathbf{q} =\gamma^{-1}(\mathbf{R})$ ($\gamma$ is defined in Section \ref{sec:supp_basis}) follows a Bingham distribution, i.e., the matrix Fisher distribution is a reparameterization of the Bingham distribution. 

In this section, we derive the fact of the equivalence of Bingham distribution and matrix Fisher distribution and clarify the relationships between the various parameters.

In measure theory, the \textit{Lebesgue measure} \cite{enwiki:lebesgue} assigns a measure to subsets of n-dimensional Euclidean space, and the \textit{Haar measure} \cite{enwiki:haar} assigns an ``invariant volume'' to subsets of locally compact topological groups, in our case, the Lie group $\SO$. We define $\mathrm{d}\mathbf{q}$ based on Lebesgue measure and $\mathrm{d}\mathbf{R}$ based on Haar measure.

\begin{proposition}
\label{prop:constant_d}
The scaling factor from unit quaternions to rotation matrices is constant, and satisfies
\begin{equation}
    \label{eq:equ_d_2pi2}
    \mathrm{d}\mathbf{R} = \frac{1}{2\pi^2} \mathrm{d}\mathbf{q}
\end{equation}
\end{proposition}

\begin{proof}
Define $S$ as the Lebesgue measure on $\mathcal{S}^3$ and $T$ as the Haar measure on $\SO$. Generally we can write 
\begin{equation*}
    T(\mathrm{d}\mathbf{R}) = \alpha(\mathbf{q})S(\mathrm{d}\mathbf{q})
\end{equation*}
where $\alpha(\mathbf{q})$ is the scaling factor from unit quaternions to rotation matrices, or specifically,
\begin{equation}
\label{eq:alpha_q}
\begin{aligned}
    T(\mathrm{d}\mathbf{R_1}) &= \alpha(\mathbf{q_1})S(\mathrm{d}\mathbf{q_1})\\
    T(\mathrm{d}\mathbf{R_2}) &= \alpha(\mathbf{q_2})S(\mathrm{d}\mathbf{q_2})
\end{aligned}
\end{equation}
Due to the invariance of measure $S$ on $\mathcal{S}^3$, we have
\begin{equation}
    \label{eq:equ_dq}
    S(\mathrm{d}\mathbf{q_1}) = S(\mathrm{d}\mathbf{q_2})
\end{equation}
Define $\nu$ as the mapping from $\mathcal{S}^3$ to $\SO$, \textit{i.e.}, $\mathrm{d}\mathbf{R} = \nu(\mathrm{d}\mathbf{q})$.
Define $\mathbf{h}$ as an element in $\mathcal{S}^3$ satisfying
\begin{equation*}
    \mathbf{h} \mathrm{d}\mathbf{q_1}  = \mathrm{d}\mathbf{q_2}
\end{equation*}
we then induce $\mathbf{\hat{h}}=\nu\circ \mathbf{h}\circ \nu^{-1}$ which is an element in $\SO$, which thus satisfies
\begin{equation*}
    \mathbf{\hat{h}}\nu\left(\mathrm{d}\mathbf{q_{1}}\right)=\nu\left( \mathrm{d}\mathbf{q_{2}}\right)
\end{equation*}
Due to the invariance of measure $T$ on $\SO$ \cite{enwiki:haar}, we have
\begin{equation*}
    T(\mathbf{\hat{h}}\nu\left( \mathrm{d}\mathbf{q_{1}}\right))=
    T\left(\nu\left( \mathrm{d}\mathbf{q_{1}}\right)\right)=T\left(\nu\left(\mathrm{d}\mathbf{q_{2}}\right)\right)
\end{equation*}
\textit{i.e., }
\begin{equation}
    \label{eq:equ_dr}
    T\left(\mathrm{d}\mathbf{R_1}\right) = T\left(\mathrm{d}\mathbf{R_2}\right)
\end{equation}
Considering arbitrary $\mathrm{d}\mathbf{q_{1}}$ and $\mathrm{d}\mathbf{q_{2}}$, and based on Eq. \ref{eq:alpha_q}, \ref{eq:equ_dq} and \ref{eq:equ_dr}, we can derive that $\alpha(\mathbf{q})$ is a constant, \textit{i.e.}, 
\begin{equation}
    \label{eq:equ_d_constant}
    \mathrm{d}\mathbf{R}=\alpha \mathrm{d}\mathbf{q}.
\end{equation}

Known that the Haar measure is uniquely specified by adding the normalization condition \cite{enwiki:haar}, we have
\begin{equation*}
    \int_{\mathbf{R} \in \SO} \mathrm{d} \mathbf{R}=1
\end{equation*}
and based on the definition of unit quaternions,
\begin{equation*}
    \oint_{\mathbf{q} \in \mathcal{S}^{3}} \mathrm{d} \mathbf{q}=\left|\mathcal{S}^3\right| = 2\pi^2
\end{equation*}
According to Eq. \ref{eq:equ_d_constant}, we can derive that
\begin{equation*}
    \mathrm{d}\mathbf{R} = \frac{1}{2\pi^2} \mathrm{d}\mathbf{q}
\end{equation*}
as claimed.
\end{proof}

Let $\mathbf{I}_n$ be the n-dimensional identity matrix, and $\mathbf{\epsilon}_i, i=1, 2, \dots, n$ be the columns of $\mathbf{I}_n$.
Let $\mathbf{E}_i=2\mathbf{\epsilon}_i\mathbf{\epsilon}_i^T-\mathbf{I}_3, i=1,2,3$ and $\mathbf{E}_4=\mathbf{I}_3$.
Define $Q(\mathbf{X})$ for a $3\times 3$ rotation matrix as 
\begin{equation}
\label{eq:3.3}
    4Q(\mathbf{X}) = 4\mathbf{x}\mathbf{x}^T-\mathbf{I}_4
\end{equation}
where $\mathbf{x} = \gamma^{-1}(\mathbf{X})$.
Apply \textit{proper} singular value decomposition \cite{lee2018bayesian, mohlin2020probabilistic} to $\mathbf{A}$ as Eq. \ref{eq:supp_proper}
\begin{equation*}
    \mathbf{A} = \mathbf{U}\mathbf{S}\mathbf{V}^T
\end{equation*}
where $\mathbf{U}$ and $\mathbf{V}$ are guaranteed to be rotation matrices and $\mathbf{S}$ contains the \textit{proper} singular values with $s_{1} \geq s_{2} \geq\left|s_{3}\right|$.
Define $T(\mathbf{A})$ for any real $3\times 3$ matrix $\mathbf{A}$ as 
\begin{equation}
    \label{eq:3.4}
    4T(\mathbf{A}) = \sum_{i=1}^{4} z_i Q(\mathbf{UE_iV}).
\end{equation}
Let $z_1, z_2, z_3, z_4$ denote the entries of $\mathbf{Z}$ and $\mathbf{m}_1, \mathbf{m}_2, \mathbf{m}_3, \mathbf{m}_4$ denote the columns of $\mathbf{M}$.

\begin{proposition}
Suppose the parameters satisfy the following relationships
\begin{equation}
    \label{eq:z&s}
    \mathbf{Z} = 4T(\mathbf{S})
\end{equation}
\begin{equation}
    \mathbf{m}_i=\gamma^{-1}(\mathbf{U}\mathbf{E}_i\mathbf{V}^T), i=1,2,3,4
\end{equation}
and the inputs
\begin{equation*}
    \mathbf{R}= \gamma(\mathbf{q}),
\end{equation*}
matrix Fisher distribution is equivalent to Bingham distribution with the relationship
\begin{equation}
    \operatorname{tr}(\mathbf{A} \mathbf{R}^T) = \mathbf{q}^T \mathbf{M} \mathbf{Z} \mathbf{M}^T \mathbf{q}
\end{equation}
and
\begin{equation}
    \label{eq:equ_p_2pi2}
    p_F(\mathbf{R}) = 2\pi^2 p_B(\mathbf{q})
\end{equation}
\end{proposition}

\begin{proof}
Assume $\mathbf{S}=\operatorname{diag}(s_1, s_2, s_3)$ then we may write 
\begin{equation*}
    4\mathbf{A}=\sum_{i=1}^4 z_i\mathbf{U}\mathbf{E}_i\mathbf{V}^T
\end{equation*}
uniquely, with
\begin{equation*}
\begin{aligned}
    z_1&=s_1 - s_2 - s_3 \\
    z_2 &= s_2 - s_1 - s_3 \\
    z_3 &= s_3 - s_1 - s_2 \\
    z_4 &= - z_1 - z_2 - z_3.
\end{aligned}
\end{equation*}
Also, since $4\mathbf{E}_i = 3\mathbf{E}_i -\sum_j \mathbf{E}_j, i \neq j$, 
Eq. \ref{eq:3.4} agrees with Eq. \ref{eq:3.3} on SO(3).
Assmue $\gamma(\mathbf{m}_i)=\mathbf{U}\mathbf{E}_i\mathbf{V}^T, i=1,2,3,4$, then $\mathbf{m}_i$ are mutually orthogonal, since $\operatorname{tr}\left(\gamma(\mathbf{m}_i)\gamma(\mathbf{m}_j)^T\right)=-1$ if $i\neq j$. Hence we may write
\begin{equation*}
    4T(\mathbf{A})=\mathbf{M}\mathbf{Z}\mathbf{M}^T
\end{equation*}
where $\mathbf{Z}=\operatorname{diag}(z_1, z_2, z_3, z_4)$ has a zero trace and $\mathbf{M} = (\mathbf{m}_1, \mathbf{m}_2, \mathbf{m}_3, \mathbf{m}_4)$ in SO(4).
Note that 
\begin{equation*}
4T(\mathbf{S})=\mathbf{Z}
\end{equation*}
and 
\begin{equation*}
    4\operatorname{tr}(\mathbf{A} \mathbf{R}^T) = \sum_{i=1}^4 z_i \operatorname{tr}(\mathbf{U}\mathbf{E}_i\mathbf{V}^T\mathbf{R}^T),
\end{equation*}
we have
\begin{equation}
    \label{eq:equ_exponent}
    \operatorname{tr}(\mathbf{A}\mathbf{R}^T) = \mathbf{q}^T\mathbf{M}\mathbf{Z}\mathbf{M}^T\mathbf{q}
\end{equation}

Due to the scaling factor from unit quaternions to rotation matrices is constant (See Prop. \ref{prop:constant_d}),
matrix Fisher distribution is equivalent to Bingham distribution.
Based on Eq.  \ref{eq:equ_exponent} and \ref{eq:equ_d_2pi2}, and the conservation of the total probability, it can be shown that
\begin{equation*}
    p_F(\mathbf{R}) = 2\pi^2 p_B(\mathbf{q})
\end{equation*}
as claimed.
\end{proof}

Note that the proposition can also be verified by the relationships between the normalization constant $F_B(\mathbf{Z})$ and $F_F(\mathbf{A})$.
As discussed in \cite{lee2018bayesian, mohlin2020probabilistic, deng2020deep}, when $\mathbf{Z}$ satisfies Eq. \ref{eq:z&s}, the constant
\begin{equation*}
\begin{aligned}
    F_B(\mathbf{Z})=\oint_{\mathbf{q} \in \mathbb{S}^{3}} \exp \left(\mathbf{q}^{T} \mathbf{M Z M}^{T} \mathbf{q}\right) \mathrm{d} \mathbf{q}&=\left|\mathbb{S}^3\right|{ }_{1} F_{1}\left(\frac{1}{2}, 2, \mathbf{Z}\right) \\
    &=2\pi^2{ }_{1} F_{1}\left(\frac{1}{2}, 2, \mathbf{Z}\right)
\end{aligned}
\end{equation*}
and
\begin{equation*}
\begin{aligned}
    F_F(\mathbf{A})=\int_{\mathbf{R} \in \SO} \exp \left(\operatorname{tr}\left(\mathbf{A}^{T} \mathbf{R}\right)\right) \mathrm{d} \mathbf{R}={ }_{1} F_{1}\left(\frac{1}{2}, 2, \mathbf{Z}\right)
\end{aligned}
\end{equation*}
where ${ }_{1} F_{1}(\cdot, \cdot, \cdot)$ is the generalized hypergeometric function \cite{koev2006efficient} of a matrix argument. So 
\begin{equation*}
\label{eq:constant}
    F_F(\mathbf{Z}) = \frac{1}{2\pi^2} F_F(\mathbf{A}).
\end{equation*}
Considering Eq. \ref{eq:equ_exponent}, we have
\begin{equation*}
    p_F(\mathbf{R}) = 2\pi^2 p_B(\mathbf{q})
\end{equation*}

%% file: tex_supp/math_constant1.tex
\subsection{Normalization Constant of Matrix Fisher Distribution}
We follow \cite{mohlin2020probabilistic} to compute the normalization constant.
As pointed in \cite{lee2018bayesian}, the normalizing constant of matrix Fisher distribution can be expressed as a one dimensional integral over Bessel functions as
\begin{equation*}
\begin{aligned}
c(S)=& \int_{-1}^{1} \frac{1}{2} I_{0}\left[\frac{1}{2}\left(s_{i}-s_{j}\right)(1-u)\right] \\
& \times I_{0}\left[\frac{1}{2}\left(s_{i}+s_{j}\right)(1+u)\right] \exp \left(s_{k} u\right) d u
\end{aligned}
\end{equation*}
and 
\begin{equation*}
\begin{aligned}
\frac{\partial c(S)}{\partial s_{i}}=& \int_{-1}^{1} \frac{1}{4}(1-u) I_{1}\left[\frac{1}{2}\left(s_{i}-s_{j}\right)(1-u)\right] \\
& \times I_{0}\left[\frac{1}{2}\left(s_{i}+s_{j}\right)(1+u)\right] \exp \left(s_{k} u\right) \\
&+\frac{1}{4}(1+u) I_{0}\left[\frac{1}{2}\left(s_{i}-s_{j}\right)(1-u)\right] \\
& \times I_{1}\left[\frac{1}{2}\left(s_{i}+s_{j}\right)(1+u)\right] \exp \left(s_{k} u\right) d u \end{aligned}
\end{equation*}

\begin{equation*}
\begin{aligned}
\frac{\partial c(S)}{\partial s_{j}}=& \int_{-1}^{1}-\frac{1}{4}(1-u) I_{1}\left[\frac{1}{2}\left(s_{i}-s_{j}\right)(1-u)\right] \\
& \times I_{0}\left[\frac{1}{2}\left(s_{i}+s_{j}\right)(1+u)\right] \exp \left(s_{k} u\right) \\
&+\frac{1}{4}(1+u) I_{0}\left[\frac{1}{2}\left(s_{i}-s_{j}\right)(1-u)\right] \\
& \times I_{1}\left[\frac{1}{2}\left(s_{i}+s_{j}\right)(1+u)\right] \exp \left(s_{k} u\right) d u
\end{aligned}
\end{equation*}

\begin{equation*}
\begin{aligned}
\frac{\partial c(S)}{\partial s_{k}}=& \int_{-1}^{1} \frac{1}{2} I_{0}\left[\frac{1}{2}\left(s_{i}-s_{j}\right)(1-u)\right] \\
& \times I_{0}\left[\frac{1}{2}\left(s_{i}+s_{j}\right)(1+u)\right] u \exp \left(s_{k} u\right) d u
\end{aligned}
\end{equation*}
for any $(i, j, k) \in \mathcal{I}$.

We approximate this integral using the trapezoid rule, where in experiments, 511 trapezoids are used. We use standard polynomials to approximate the Bessel function using Horner’s method. 

Please see Section 5 of \cite{mohlin2020probabilistic}'s supplementary for more details.

%% file: tex_supp/table_100.tex
\setlength{\tabcolsep}{12pt}
\begin{table}[t]
\footnotesize
  \centering
  \caption{\textbf{Comparing our proposed FisherMatch with the baselines on ModelNet10-SO(3) dataset under 100\% labeled data.}}
    \begin{tabular}{c|l|cc}
    \multirow{2}{*}{Category} & \multicolumn{1}{c}{\multirow{2}{*}{Method}} & \multicolumn{2}{|c}{100\%} \\
\cmidrule{3-4}          &       & \multicolumn{1}{l}{Mean$\downarrow$} & \multicolumn{1}{l}{Med.$\downarrow$} \\
    \midrule
    \multirow{4}{*}{Sofa} & Sup.-L1 & 19.28 & 6.64 \\
                & Sup.-Fisher & 18.62 & 5.77 \\
                & SSL-L1-Consist. &   17.18    &  5.27 \\
                & SSL-FisherMatch &    \textbf{14.37}   &  \textbf{4.32}  \\
    \midrule
    \multirow{4}{*}{Chair} & Sup.-L1 & 17.65 & 7.48 \\
                & Sup.-Fisher & 17.38 & 6.78 \\
                & SSL-L1-Consist. &   14.78    & 6.19 \\
                & SSL-FisherMatch &   \textbf{13.01}    &  \textbf{5.35} \\
    \bottomrule
    \end{tabular}
  \label{tab:100}
\end{table}

%% file: tex_supp/table_objectron.tex
\setlength{\tabcolsep}{12pt}
\begin{table}[t]
\footnotesize
  \centering
  \caption{\textbf{Comparing our proposed FisherMatch with the baselines on Objectron  dataset with 1\% labeled data.}}
    \begin{tabular}{c|l|cc}
    \multirow{2}{*}{Category} & \multicolumn{1}{c}{\multirow{2}{*}{Method}} & \multicolumn{2}{|c}{1\%} \\
\cmidrule{3-4}          &       & \multicolumn{1}{l}{Mean$\downarrow$} & \multicolumn{1}{l}{Med.$\downarrow$} \\
    \midrule
    \multirow{5}{*}{Bike} 
                & Sup.-L1     & 53.6 & 21.2 \\
                & Sup.-Fisher & 51.2 & 24.0 \\
                & SSL-L1-Consist. &     38.0           &    14.3        \\
                & SSL-FisherMatch &    \textbf{36.0}   &  \textbf{13.8}  \\
    \cmidrule{2-4}       & Full sup. & 26.7   & 9.7   \\
    \midrule
    \multirow{5}{*}{Camera} 
                & Sup.-L1     & 46.0  & 22.8  \\
                & Sup.-Fisher & 39.0 & 18.7 \\
                & SSL-L1-Consist. &    40.9            &    19.0       \\
                & SSL-FisherMatch &   \textbf{33.6}    &  \textbf{15.9} \\
    \cmidrule{2-4}       & Full sup. & 24.4   & 9.5   \\
    \bottomrule
    \end{tabular}
  \label{tab:objectron}
\end{table}